\documentclass[11pt]{article}

\usepackage{amsmath,amssymb,amsthm}
\usepackage{dsfont}
\usepackage{url}
\usepackage{hyperref}
\usepackage[round,comma]{natbib}
\usepackage{graphicx}
\usepackage{fullpage}

\usepackage{color}

\newcommand\vknote[1]{{\color{blue} [[ #1 ]]}}

\def\ie{{\it i.e.}~}
\def\eg{{\it e.g.}~}

\def\F{{\mathcal F}}
\def\G{{\mathcal G}}
\def\opt{\mathrm{opt}}
\def\sign{\operatorname{sign}}
\def\NS{\operatorname{NS}}

\def\taumax{\tau_{\mbox{\footnotesize{max}}}}
\def\one{\mathds{1}}

\newcommand{\ip}[2]{\ensuremath \langle #1, #2 \rangle}
\newcommand{\twonorm}[1]{\ensuremath \Vert #1 \Vert_{2}}
\newcommand{\infinitynorm}[1]{\ensuremath \Vert #1 \Vert_{\infty}}
\newcommand{\operatornorm}[1]{\ensuremath \Vert #1 \Vert_{op}}
\newcommand{\frobenius}[1]{\ensuremath \Vert #1 \Vert_{F}}
\newcommand{\eat}[1]{}

\def\scA{{\mathcal A}}
\def\scE{{\mathcal E}}
\def\scG{{\mathcal G}}
\def\EX{\mathsf{EX}}
\def\E{\mathbb{E}}
\def\moo{{\{-1, 1\}}}
\def\reals{\mathbb{R}}
\def\one{\mathds{1}}
\def\OS{\mathsf{OS}}

\def\junta{{\mathcal J}}

\def\yy{{\mathcal Y}}
\def\poly{\mathrm{poly}}
\def\Nu{{\mathcal V}}
\def\transpose{T}
\def\th{{^{\textit{th}}}}

\newcommand{\dtv}[1]{{ \Vert #1 \Vert_{\mbox{\footnotesize{TV}}}}}
\newcommand{\pinorm}[1]{{ \Vert #1 \Vert_{\pi}}}

\newtheorem{lemma}{Lemma}
\newtheorem{proposition}{Proposition}
\newtheorem{corollary}{Corollary}
\newtheorem{theorem}{Theorem}
\newtheorem{remark}{Remark}
\newtheorem{definition}{Definition}

\title{MCMC Learning}
\author{Varun Kanade\thanks{This work was performed while the author was at the University of California, Berkeley and at the Simons Institute, Berkeley} \\ \'{E}cole normale sup\'{e}rieure \\
	\texttt{varun.kanade@ens.fr} \and Elchanan Mossel\thanks{Supported NSF grants DMS 1106999 and CCF 1320105, ONR grant number N00014-14-1-0823
and grant 328025 from the Simons Foundation} \\ University of Pennsylvania and University of California, Berkeley \\ \texttt{mossel@stat.berkeley.edu}}

\begin{document}

\maketitle

\begin{abstract}
The theory of learning under the uniform distribution is rich and deep, with
connections to cryptography, computational complexity, and the analysis of
boolean functions to name a few areas. This theory however is very limited
due to the fact that the uniform distribution and the corresponding Fourier
basis are rarely encountered as a statistical model.

A family of distributions that vastly generalizes the uniform distribution on
the Boolean cube is that of distributions represented by Markov Random Fields
(MRF). Markov Random Fields are one of the main tools for modeling high
dimensional data in many areas of statistics and machine learning.

In this paper we initiate the investigation of extending central ideas, methods
and algorithms from the theory of learning under the uniform distribution to the
setup of learning concepts given examples from MRF distributions. In particular,
our results establish a novel connection between properties of MCMC sampling of
MRFs and learning under the MRF distribution. 

\end{abstract}

\section{Introduction}
\label{sec:intro}
The theory of learning under the uniform distribution is well developed and has
rich and beautiful connections to discrete Fourier analysis, computational
complexity, cryptography and combinatorics to name a few areas. However, these
methods are very limited since they rely on the assumption that examples are
drawn from the uniform distribution over the Boolean cube or other product
distributions. In this paper we make a first step in extending ideas,
techniques and algorithms from this theory to a much broader family of
distributions, namely, to Markov Random Fields. 

\subsection{Learning Under the Uniform Distribution}

Since the seminal work of \citet{LMN:1993}, the study of learning under the
uniform distribution has developed into a major area of research; the principal
tool is the simple and explicit Fourier expansion of functions defined on the
boolean cube ($\moo^n$):
\[ f(x) = \sum_{S \subseteq [n]} \hat{f}(S) \chi_S(x), \quad \chi_S(x) =
\prod_{i \in S} x_i.  \]
This connection allows a rich class of algorithms that are based on learning
coefficients of $f$ for several classes of functions.  Moreover, this
connection allows application of sophisticated results in the theory of Boolean
functions including hyper-contractivity, number theoretic properties and
invariance, \eg~\citep{OS:2007,ShpilkaTal:11,KOS:2002}.  On the other hand, the
central role of the uniform distribution in computational complexity and
cryptography relates learning under the uniform distribution to key themes in
theoretical computer science including de-randomization, hardness and
cryptography, \eg~\citep{Kharitonov:1993,NR:2004,D-SLMSWW:2008}.

Given the elegant theoretical work in this area, it is a little disappointing
that these results and techniques impose such stringent assumptions on the
underlying distribution. The assumption of independent examples sampled from
the uniform distribution is an idealization that would rarely, if ever, be
applicable in practice. In \emph{real} distributions, features are correlated
and correlations deem the analysis of algorithms that assume independence
useless. Thus, it is worthwhile to ask the following question: \medskip \\
\noindent {\bf Question 1:}
{\em Can the Fourier Learning Theory extend to correlated features?}

\subsection{Markov Random Fields} 

Markov random fields are a standard way of representing high dimensional
distributions (see \eg~\citep{KS:1980}). Recall that a Markov random field on a
finite graph $G=(V,E)$ and taking values in a discrete set $\scA$, is a
probability distribution on $\scA^V$ of the form $\Pr[(\sigma_v)_{v \in V}] =
Z^{-1} \prod_{C} \phi_C((\sigma_v)_{v \in C})$, where the product is over all
cliques $C$ in the graph, $\phi_C$ are some non-negative valued functions and
$Z$ is the normalization constant. Here $(\sigma_v)_{v \in V}$ is an assignment
from $V \rightarrow \scA$.

Markov Random Fields are widely used in vision, computational biology,
biostatistics, spatial statistics and several other areas. The popularity of
Markov Random Fields as modeling tools is coupled with extensive algorithmic
theory studying sampling from these models, estimating their parameters and
recovering them.
However, to the best of our knowledge the following question has not been
studied. \medskip \\
\noindent {\bf Question 2:}
{\em  For an unknown function $f : \scA^V \to \{-1,1\}$ from a class $\F$ and labeled samples
from the Markov Random Field, can we learn the function?} \medskip \\ 
Of course the problem stated above is a special case of learning a function
class given a general distribution~\citep{Valiant:1984, KV:1994}. Therefore, a
learning algorithm that can be applied for a general distribution can be also
applied to MRF distributions. However, the real question that we seek to ask
above is the following: {\em Can we utilize the structure of the MRF to obtain
better learning algorithms?} 

\subsection{Our Contributions} 

In this paper we begin to provide an answer to the questions posed above. We show
how methods that have been used in the theory of learning under the uniform
distribution can be also applied for learning from certain MRF distributions. 

This may sound surprising as the theory of learning under the uniform
distribution strongly relies on the explicit Fourier representation of
functions. Given an MRF distribution, one can also imagine expanding a function
in terms of a \emph{Fourier basis} for the MRF, the eigenvectors of the
transition matrix of the Gibbs Markov Chain associated with the MRF, which are
orthogonal with respect to the MRF distribution.  It seems however that this
approach is na\"{i}ve since: 
\begin{enumerate}
\item[(a)] Each eigenvector is of size $|\scA|^{|V|}$; how does one store them?
\item[(b)] How does one find these eigenvectors?
\item[(c)] How does one find the expansion of a function in terms of these
eigenvectors? 
\end{enumerate}

\noindent {\bf MCMC Learning}: 
The main effort in this paper is to provide an answer to the questions above.
For this we use Gibbs sampling, which is a Markov chain Monte Carlo (MCMC)
algorithm that is used to sample from an MRF. We will use this MCMC method as
the main engine in our learning algorithms. The Gibbs MC is reversible and
therefore its eigenvectors are orthogonal with respect to the MRF distribution.
Also, the sampling algorithm is straightforward to implement given access to
the underlying graph and potential functions. There is a vast literature
studying the convergence rates of this sampling algorithm; our results require
that the Gibbs samplers are rapidly mixing.
%

In Section~\ref{sec:eigen}, we show how the eigenvectors of the transition
matrix of the Gibbs MC can be computed implicitly. We focus on the eigenvectors
corresponding to the higher eigenvalues.  These eigenvectors correspond to the
stable part of the spectrum, \ie the part that is not very sensitive to small
perturbation. Perhaps surprisingly, despite the exponential size of the matrix,
we show that it is possible to adapt the power iteration method to this
setting. 

A function from $\scA^{V} \rightarrow \reals$ can be viewed as a $|\scA|^{|V|}$
dimensional vector and thus applying powers of the transition matrix to it
results in another function from $\scA^V \rightarrow \reals$. Observe that the
powers of a transition matrix define distributions in time over the state space
of the the Gibbs MC. Thus, the value of the function obtained by applying
powers of a transition matrix can be approximated by sampling using the Gibbs
Markov chain. Our main technical result (see Theorem~\ref{thm:spectrum}) shows
that any function approximated by ``top'' eigenvectors of the transition matrix
of the Gibbs MC can be expressed a linear combination of powers of the the
transition matrix applied to a suitable collection of ``basis'' functions,
whenever certain technical conditions hold.

%
The reason for focusing on the part of the spectrum corresponding to stable
eigenvectors is twofold.  First, it is technically easier to access this part of
the spectrum. Furthermore, we think of eigenvectors corresponding to small
eigenvalues as unstable. Consider Gibbs sampling as the true temporal evolution
of the system and let $\nu$ be an eigenvector corresponding to a small
eigenvalue.  Then calculating $\nu(x)$ provides very little information on
$\nu(y)$ where $y$ is obtained from $x$ after a short evolution of the Gibbs
sampler. The reasoning just applied is a generalization of the classical
reasoning for concentrating on the low frequency part of the Fourier expansion
in traditional signal processing. \medskip 

\noindent{\bf Noise Sensitivity and Learning}: In the case of the uniform
distribution, the noise sensitivity (with parameter $\epsilon$) of a boolean
function $f$, is defined as the probability that $f(x) \neq f(y)$, where $x$ is
chosen uniformly at random and $y$ is obtained from $x$ by flipping each bit
with probability $\epsilon$. \citet{KOS:2002} gave an elegant characterization
of learning in terms of noise sensitivity. Using this characterization, they
showed that intersections and thresholds of halfspaces can be elegantly learned
with respect to the uniform distribution. In
Section~\ref{sec:noise-sensitivity}, we show that the notion of noise
sensitivity and the results regarding functions with low noise sensitivity can
be generalized to MRF distributions. \medskip

\noindent {\bf Learning Juntas}: We also consider the so-called junta learning
problem. A junta is a function that depends only on a small subset of the
variables. Learning juntas from i.i.d. examples is a notoriously difficult
problem, see~\citep{Blum:1992, MOS:2004}. However, if the learning algorithm has
access to \emph{labeled} examples that are received from a Gibbs sampler, these
correlated examples can be useful for learning juntas. We show that under
standard technical conditions on the Gibbs MC, juntas can be learned in
polynomial time by a very simple algorithm. These results are presented in
Section~\ref{sec:juntas}. \medskip

\noindent {\bf Relation to Structure Learning}: In this paper, we assume that
learning algorithms have the ability to sample from the Gibbs Markov Chain
corresponding to the MRF. While such data would be hard to come by in practice,
we remark that there is a vast literature regarding learning the structure and
parameters of MRFs using \emph{unlabeled} data and that it has recently been
established that this can be done efficiently under very general conditions
\citep{Bresler:2014}.  Once the structure of the underlying MRF is known, Gibbs
sampling is an extremely efficient procedure. Thus, the methods proposed in
this work could be used in conjunction with the techniques for MRF structure
learning. The eigenvectors of the transition matrix could be viewed as
\emph{features} for learning, thus the methods proposed in this paper can be
viewed as feature learning.

\subsection{Related Work} 

The idea of considering Markov Chains or Random Walks in the context of
learning is not new.  However, none of the results and models considered before
give non-trivial improvements or algorithms in the context of MRFs.  Work of
\citet{AV:1995} studies a Markov chain based model where the main interest was
in characterizing the number of new nodes visited.  \citet{Gamarnik:1999}
observed that after the mixing time a chain can simulate i.i.d. samples from
the stationary distribution and thus obtained learning results for general
Markov chains. \citet{BFH:1994} and \citet{BMOS:2005} considered random walks
on the discrete cube and showed how to utilize the random walk model to learn
functions that cannot be easily learned from i.i.d. examples from the uniform
distribution on the discrete cube. In this same model, \citet{JW:2014} showed
that agnostic learning parities and PAC-learning thresholds of parities (TOPs)
could be performed in quasi-polynomial time.

\section{Preliminaries}
\label{sec:prelim}
Let $X$ be an instance space. In this paper, we will assume that $X$ is finite
and in particular we are mostly interested in the case when $X = \scA^n$, where
$\scA$ is some finite set. For $x, x^\prime \in \scA^n$, let $d_H(x, x^\prime)$
denote the Hamming distance between $x$ and $x^\prime$, \ie $d_H(x, x^\prime) =
|\{ i ~|~ x_i \neq x^\prime_i \}|$. 

Let $M = \langle X, P \rangle$ denote a time-reversible discrete time ergodic
Markov chain with transition matrix $P$.
When $X = \scA^n$, we say that $M$ has \emph{single-site} transitions if for
any \emph{legal} transition $x \rightarrow x^\prime$ it is the case that
$d_H(x, x^\prime) \leq 1$, \ie $P(x, x^\prime) = 0$ when $d_H(x, x^\prime) >
1$. Let $X^0 = x_0$ denote the starting state of a Markov chain $M$. Let
$P^t(x_0, \cdot)$ denote the distribution over states at time $t$, when
starting from $x_0$. Let $\pi$ denote the stationary distribution of $M$.
Denote by $\tau_M(x_0)$ the quantity: 
\[ \tau_M(x_0) = \min\{t : \dtv{P^{t}(x_0, \cdot) - \pi} \leq \frac{1}{4} \} \]
Then, define the mixing time of $M$ as $\tau_M = \max_{x_0 \in X} \tau_M(x_0)$.
We say that a Markov chain with state space $X = \scA^n$ is rapidly mixing if
$\tau_M \leq \poly(n)$.

While all the results in this paper are general, we describe two basic graphical
models that will aid the discussion. 

\subsection{Ising Model}
%
Consider a collection of nodes, $[n] = \{1, \ldots, n\}$, and for each pair $i,
j$, there is an associated interaction energy, $\beta_{ij}$. Suppose $([n], E)$
denotes the graph, where $\beta_{ij} = 0$ for $(i, j) \not\in E$.  A state
$\sigma$ of the system consists of an assignment of \emph{spins}, $\sigma_i \in
\{ +1, -1 \}$, to the nodes $[n]$. The Hamiltonian of configuration $\sigma$ is
defined as
\[ 
H(\sigma) = - \sum_{(i, j) \in E} \beta_{ij} \sigma_i \sigma_j - B \sum_{i \in [n]}
\sigma_i,
\]
where $B$ is the external field. The energy of a configuration $\sigma$ is
$\exp(-H(\sigma))$.   

The Glauber dynamics on the Ising model defines the Gibbs Markov Chain $M = \langle
\moo^n, P \rangle$, where the transitions are defined as follows:\smallskip \\
$\mbox{~~}$(i)~In state $\sigma$, pick a node $i \in [n]$ uniformly at random.
With probability $1/2$ do nothing, otherwise\smallskip \\ 
$\mbox{~~}$(ii)~Let $\sigma^\prime$ be obtained by flipping the spin at node
$i$.  Then, with probability $\exp(-H(\sigma^\prime))/(\exp(-H(\sigma) +
\exp(-H(\sigma^\prime)))\}$, the state at the next time-step is $\sigma^\prime$.
Otherwise the state at the next time-step remains unchanged.

The stationary distribution of the above dynamics is the \emph{Gibbs
distribution}, where $\pi(\sigma) \propto \exp(-H(\sigma))$. It is known that
there exists a $\beta(\Delta) > 0$ such that for all graphs of maximal degree
$\Delta$, if $ \max |\beta_{i,j}| < \beta(\Delta)$ then the dynamics above is
rapidly mixing~\citep{DobrushinShlosman:85,MosselSly:13}. 

\subsection{Graph Coloring}

Let $G = ([n], E)$ be a graph. For any $q > 0$, a \emph{valid}  $q$-coloring of
the graph $G$ is a function $C: V \rightarrow [q]$ such that for every $(i, j)
\in E$, $C(i) \neq C(j)$. For a node $i$, let $N(i) = \{j ~|~ (i, j) \in E\}$
denote the set of neighbors of $i$. Consider the Markov chain defined by the
following transition: \smallskip \\
$\mbox{~~}$(i)~In state (valid coloring) $C$, choose a node $i \in [n]$
uniformly at random. With probability $1/2$ do nothing, otherwise:\smallskip\\
$\mbox{~~}$(ii)~Let $S \subseteq [q]$ be the subset of colors defined by $S =
\{C(j) ~|~ j \in N(i) \}$. Define $C^\prime$ to be the coloring obtained by
choosing a random color $c \in [q] \setminus S$ and set $C^\prime(i) = c$,
$C^\prime(j) = C(j)$ for $j \neq i$. The state at the next time-step is
$C^\prime$.

%
The stationary distribution of the above Markov chain is uniform over the valid
colorings of the graph. It is known that the above chain is rapidly mixing when
the condition $q \geq 3 \Delta$ is satisfied, where $\Delta$ is the maximal
degree of the graph (in fact much better results are
known~\citep{Jerrum:95,Vigoda:99}).

\section{Learning Models}
\label{sec:model}
Let $X$ be a finite instance space and let $M = \langle X, P \rangle$ be an
irreducible discrete-time reversible Markov chain, where $P$ is the transition
matrix. Let $\pi_M$ denote the stationary distribution of $M$, $\tau_M$ the
mixing time. We assume that the Markov chain $M$ is \emph{rapidly mixing}, \ie
$\tau_M \leq \poly(\log(|X|))$ (note that if $X = \scA^n$, $\log(|X|) = O(n)$).

We consider the problem of learning with respect to stationary distributions of
rapidly mixing Markov chains (\eg defined by an MRF). The two graphical models
described in the previous section serve as examples of such settings. The
learning algorithm has access to the \emph{one-step} oracle, $\OS(\cdot)$, that
when queried with a state $x \in X$, returns the state after one step. Thus,
$\OS(x)$ is a random variable with distribution $P(x, \cdot)$ and can be used
to simulate the Markov chain.

Let $\F$ be a class of boolean functions over $X$. The goal of the learning
algorithm is to learn an unknown function, $f \in \F$, with respect to the
stationary distribution $\pi_M$ of the Markov chain $M$. As described above, the
learning algorithm has the ability to simulate the Markov chain using the
one-step oracle.  We will consider both PAC learning and agnostic learning. Let
$L : X \rightarrow \{-1, 1\}$ be a (possibly randomized) labeling function. In
the case of PAC learning $L$ is just the target function $f$; in the case of
agnostic learning $L$ is allowed to be completely arbitrary. Let $D$ denote the
distribution over $X \times \moo$, where for any $(x, y) \sim D$, $x \sim \pi_M$
and $y = L(x)$. \medskip \\
\noindent{\bf PAC Learning}~\citep{Valiant:1984}: In PAC learning the labeling
function is the target function $f$. The goal of the learning algorithm is to
output a hypothesis, $h : X \rightarrow \moo$, which with probability at least
$1 - \delta$ satisfies $\operatorname{err}(h) = \Pr_{x \sim \pi_M} [ h(x) \neq
f(x)] \leq \epsilon$. \medskip \\
\noindent{\bf Agnostic Learning}~\citep{KSS:1994,Haussler:1992}: In agnostic the
labeling function $L$ may be completely arbitrary. Let $D$ be the distribution
as defined above.  Let $\opt = \min_{f \in \F} \Pr_{(x, y) \sim D} [f(x) \neq
y]$.  The goal of the learning algorithm is to output a hypothesis, $h : X
\rightarrow \moo$, which with probability at least $1 - \delta$ satisfies, 
\[\operatorname{err}(h) = \Pr_{(x, y) \sim D} [h(x) \neq y] \leq \opt + \epsilon \]

Typically, one requires that the learning algorithm have time and sample
complexity that is polynomial in $n$, $1/\epsilon$ and $1/\delta$. So far, we
have not mentioned what access the learning algorithm has to labeled examples.
We consider two possible settings. \medskip

\noindent{\bf Learning with i.i.d. examples only}:
In this setting, in addition to having access to the one-step oracle,
$\OS(\cdot)$, the learning algorithm has access to the standard example oracle,
which when queried returns an example $(x, L(x))$, where $x \sim \pi_M$ and $L$
is the (possibly randomized) labeling function. \smallskip

\noindent{\bf Learning with labeled examples from MC}: In this setting, the
learning algorithm has access to a \emph{labeled} random walk, $(x^1, L(x^1)),
(x^2, L(x^2)), \ldots, $ of the Markov chain. Here $x^{i+1}$ is the (random)
state one time-step after $x^i$ and $L$ is the labeling function. Thus, the
learning algorithm can potentially exploit \emph{correlations} between
consecutive examples. \smallskip

The results in Section~\ref{sec:eigen} only require access to i.i.d. examples.
Note that these are sufficient to compute inner products with respect to the
underlying distribution, a key requirement for \emph{Fourier analysis}. The
result in Section~\ref{sec:juntas} is only applicable in the stronger setting
where the learning algorithm receives examples from a labeled Markov chain.
Note that since the chain is rapidly mixing, the learning algorithm by itself
is able to (approximately) simulate i.i.d. random examples.

%

\section{Harmonic Analysis using Eigenvectors}
\label{sec:eigen}
In this section, we show that the eigenvectors of the transition matrix can be
(approximately) expressed as linear combinations of a suitable collection of
basis functions and powers of the transition matrix applied to them.

Let $M = \langle X, P \rangle$ be a time-reversible discrete Markov chain. Let
$\pi$ be the stationary distribution of $M$. We consider the set of
right-eigenvectors of the matrix $P$.  The largest eigenvalue of $P$ is $1$ and
the corresponding eigenvector has $1$ in each co-ordinate. The left-eigenvector
in this case is the stationary distribution. For simplicity of analysis we
assume that $P(x,x) \geq 1/2$ for all $x$ which implies that all the eigenvalues of
$P$ are non-negative.  We are interested in identifying as many as possible of
the remaining eigenvectors with eigenvalues less than $1$. 

For functions, $f, g : X \rightarrow \reals$, define the inner-product, $\langle
f,g \rangle = \E_{x \sim \pi}[f(x) g(x)]$, and the norm $\Vert f \Vert_2 =
\sqrt{\langle f, f \rangle}$. Throughout this section, we will always consider
inner products and norms with respect to the distribution $\pi$.

Since $M$ is reversible, the right eigenvectors of $P$ are orthogonal with
respect to $\pi$. Thus, these eigenvectors can be used as a basis to represent
functions from $X \rightarrow \reals$.  First, we briefly show that this
approach generalizes the standard Fourier analysis on the Boolean cube, which
is commonly used in uniform-distribution learning.
 
\subsection{Fourier Analysis over the Boolean Cube}

\label{sec:unif-boolean}
Let $\moo^n$ denote the boolean cube. For $S \subseteq [n]$, the parity
function over $S$ is defined as  $\chi_{S}(x) = \prod_{i \in S} x_i$. With
respect to the uniform distribution $U_n$ over $\moo^n$, the set of parity
functions $\{ \chi_S ~|~ S \subseteq [n] \}$ form an orthonormal \emph{Fourier}
basis, \ie for $S \neq T$, $\E_{x \sim U_n}[\chi_S(x) \chi_T(x)] = 0$ and
$\E_{x \sim U_n}[\chi_S(x)^2] = 1$.

We can view the uniform distribution over $\moo^n$ as arising from the
stationary distribution of the following simple Markov chain. For $x, x^\prime$,
such that $x_i \neq x_i^\prime$ and $x_j = x^\prime_j$ for $j \neq i$, let $P(x,
x^\prime) = 1/(2n)$; $P(x, x) = 1/2$. The remaining values of the matrix $P$ are
set to $0$. This chain is rapidly mixing with mixing time $O(n \log(n))$ and the
stationary distribution is the uniform distribution over $\moo^n$. It is easy to
see and well known that every parity function $\chi_S$ is an eigenvector of $P$
with eigenvalue $1 - |S|/n$. Thus, Fourier-based learning under the uniform
distribution can be seen as a special case of Harmonic analysis using
eigenvectors of the transition matrix.


\subsection{Representing Eigenvectors Implicitly}

As in the case of the uniform distribution over the boolean cube, we would like
to find the eigenvectors of the transition matrix of a general Markov chain,
$M$, and use these as an orthonormal basis for learning. Unfortunately, in most
cases of interest explicit succinct representations of eigenvectors don't
necessarily exist and the size of the set $|X|$ is likely to be prohibitively
large, typically exponential in $n$, where $n$ is the length of the vectors in
$X$.  Thus, it is not possible to use standard techniques to obtain eigenvectors
of $P$.  Here, we show how these eigenvectors may be computed implicitly.

An eigenvector of the transition matrix $P$ is a function $\nu : X \rightarrow
\reals$. Throughout this section, we will view any function $g : X \rightarrow
\reals$ as an $|X|$-dimensional vector with value $g(x)$ at position $x$.  As
such, even writing down such a vector corresponding to an eigenvector $\nu$ is
not possible in polynomial time. Instead, our goal is to show that whenever a
suitable collection of \emph{basis functions} exists, the eigenvectors have a
simple representation in terms of these \emph{basis functions} and powers of
the transition matrix applied to them, as long as the underlying Markov chain
$M$ satisfies certain conditions. The condition we require is that the
\emph{spectrum} of the transition matrix be \emph{discrete}, \ie eigenvalues
show sharp drops.  Between these drops, the eigenvalues may be quite close to
each other, and in fact even equal. Figure~\ref{fig:spectrum} shows the
spectrum of the transition matrix of the Ising model on a cycle of $10$ nodes
for various values of $\beta$, the inverse temperature parameter. The case when
$\beta = 0$ corresponds to the uniform distribution on $\{-1, 1\}^{10}$.  One
notices that the spectrum is \emph{discrete} for small values of $\beta$
(high-temperature regime).

\begin{figure}
\begin{tabular}{cc}
\includegraphics[width=0.4\textwidth]{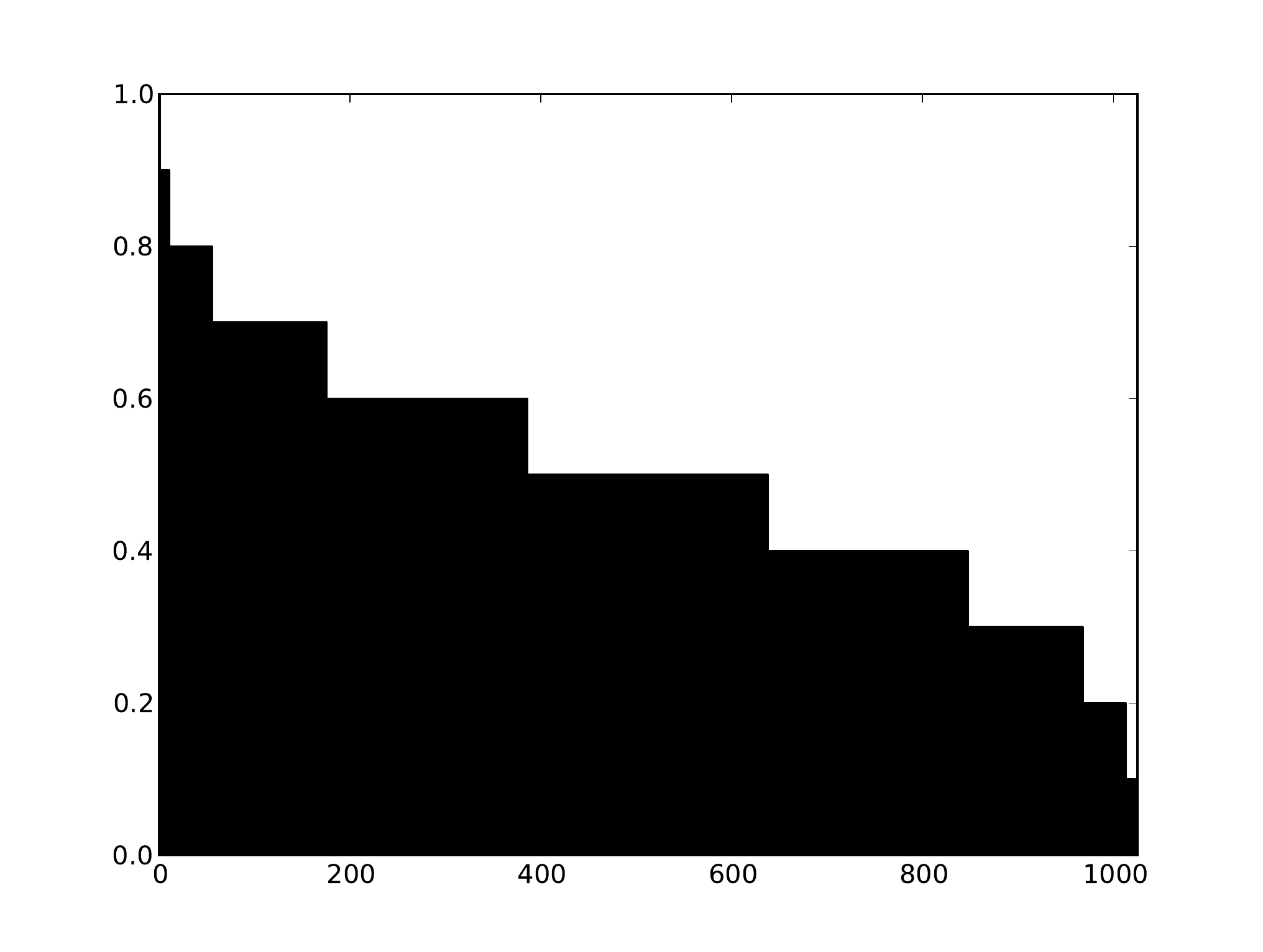} &
\includegraphics[width=0.4\textwidth]{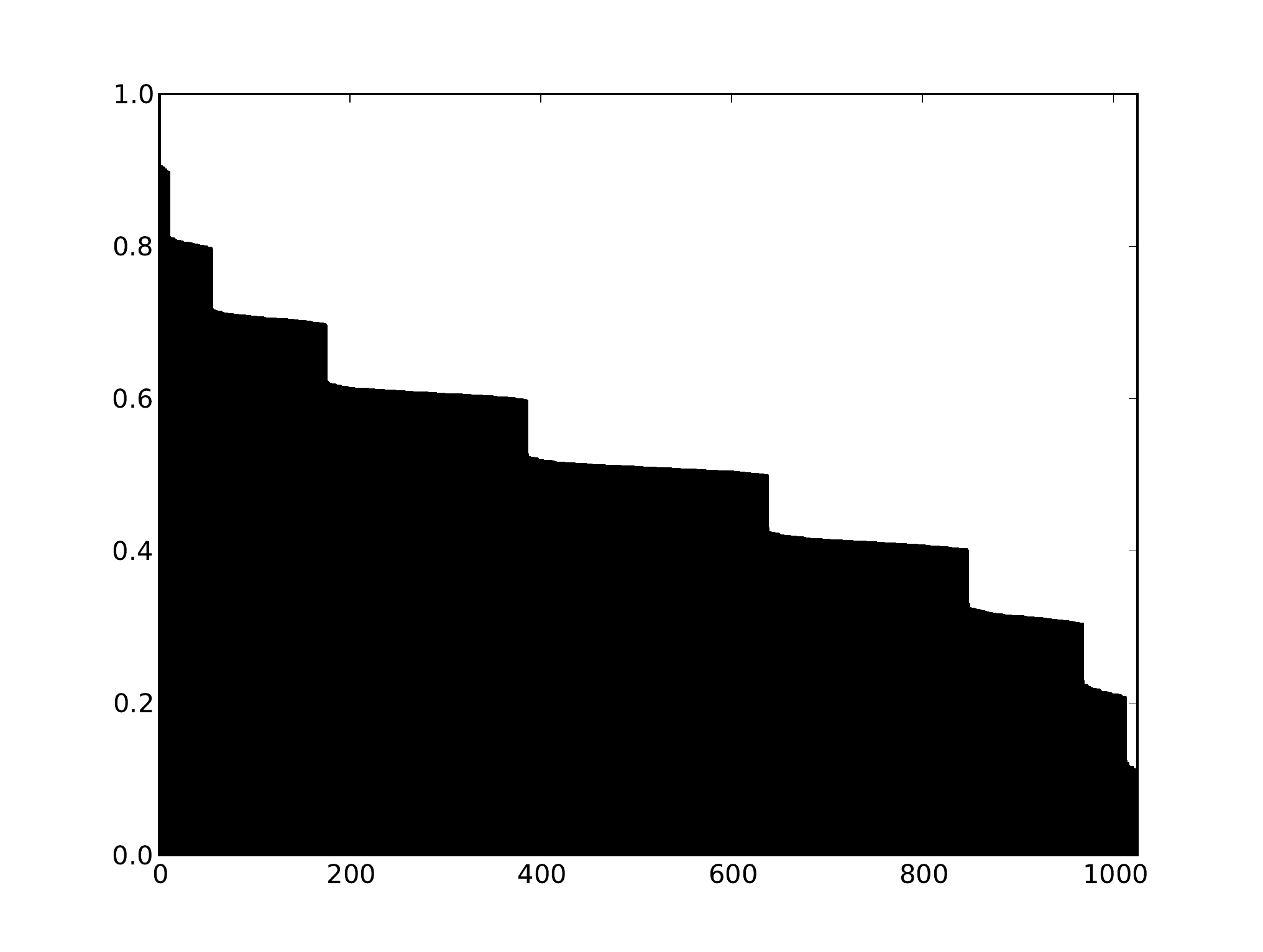} \\
$\beta = 0.00$ & $\beta=0.02$ \\
\includegraphics[width=0.4\textwidth]{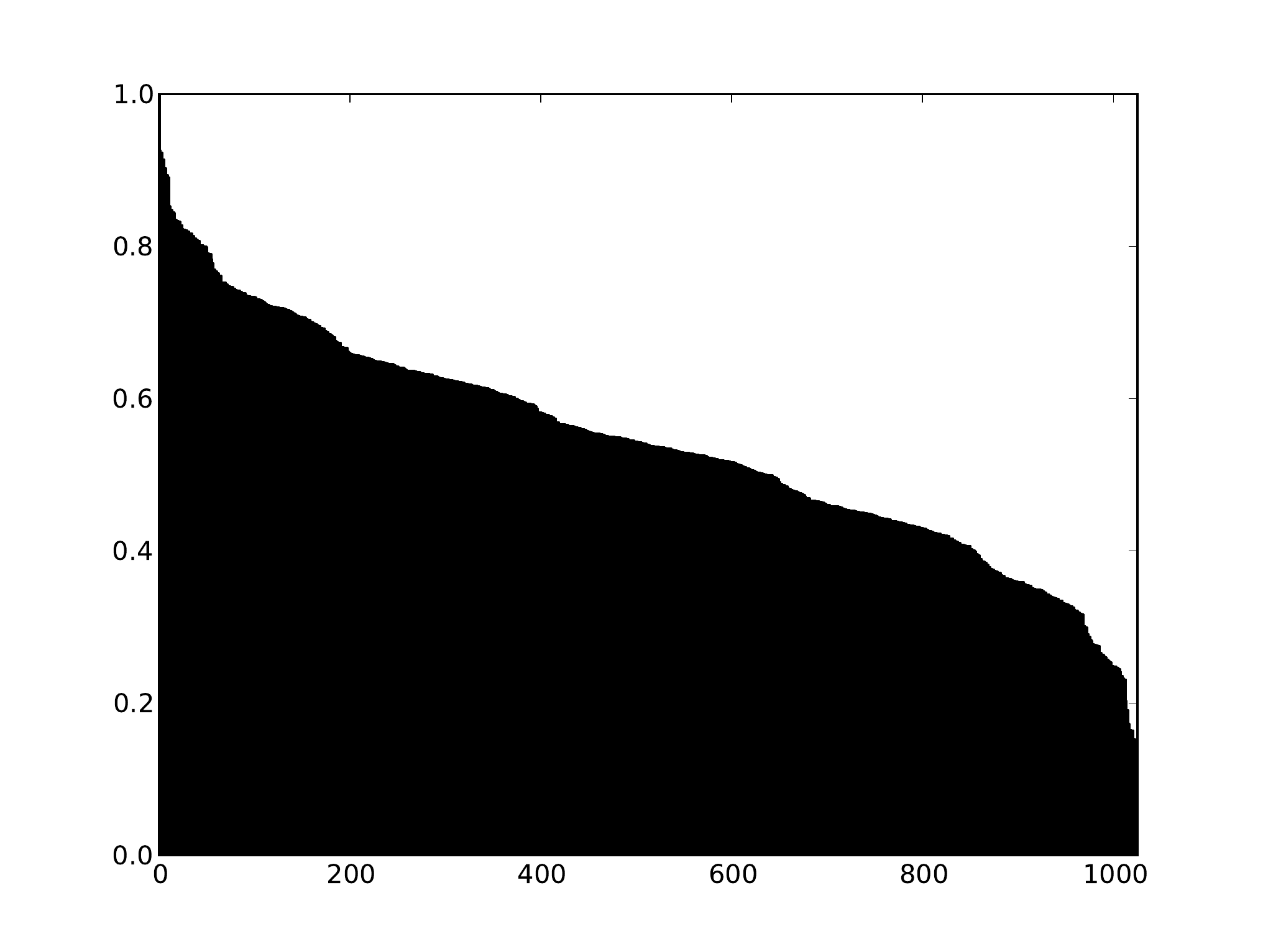} &
\includegraphics[width=0.4\textwidth]{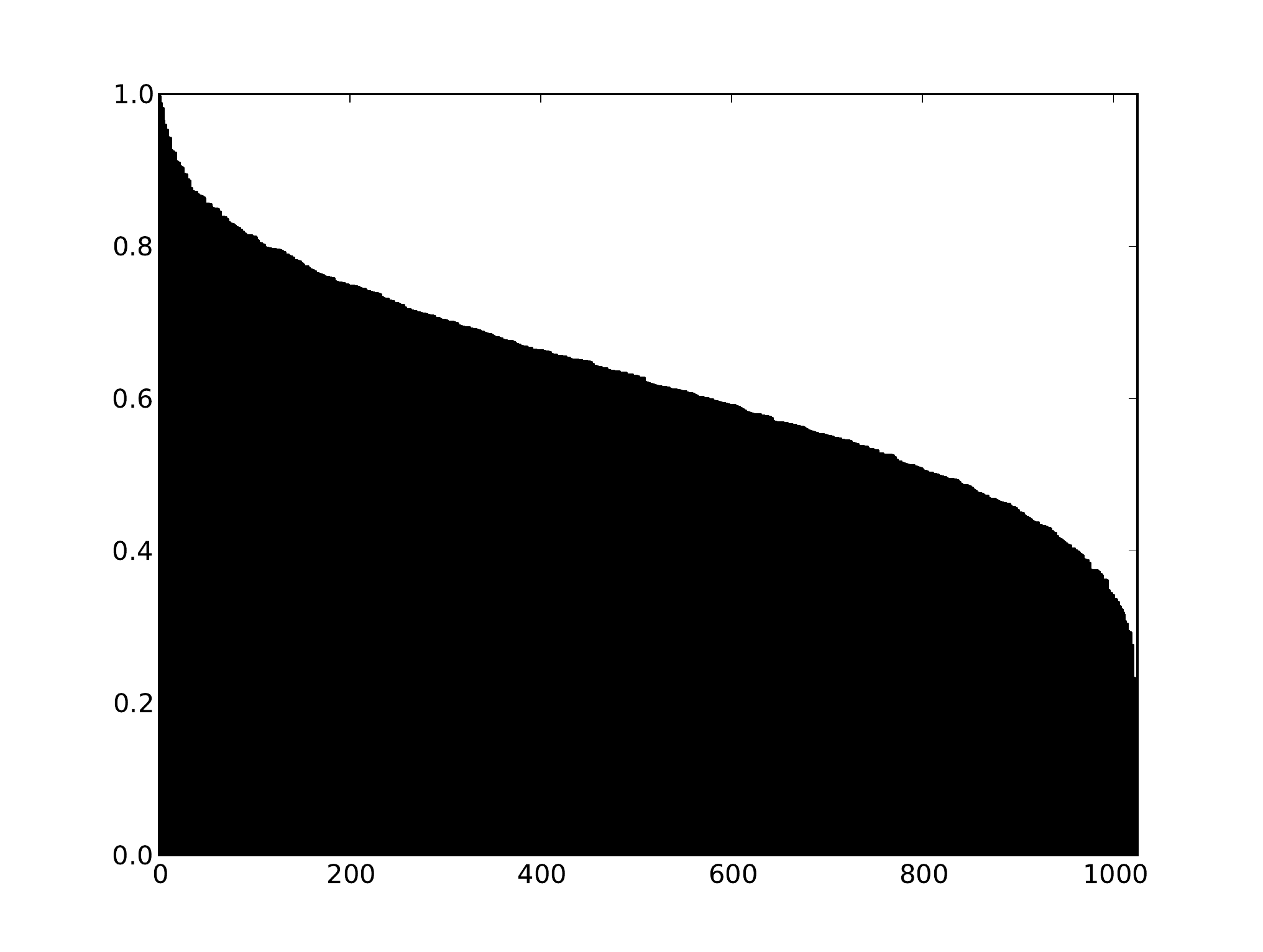} \\
$\beta = 0.1$ & $\beta=1.00$ 
\end{tabular}
\caption{\label{fig:spectrum} Spectrum of the transition matrix of the Gibbs MC
for the Ising model on a cycle of length $10$ for various values of $\beta$, the inverse temperature parameter.}
\end{figure}

Next, we formally define the requirements of a discrete spectrum.

\begin{definition}[Discrete Spectrum] \label{defn:discrete-spectrum} Let $P$ be
the transition matrix of a Markov chain and let $\lambda_1 \geq \lambda_2 \geq
\cdots \geq \lambda_i \geq \cdots \geq 0$ be the eigenvalues of $P$ in
non-increasing order. We say that $P$ has an $(N, k, \gamma, c)$-discrete
spectrum, if there exists a sequence $1 \leq i_1 \leq i_2 \leq \cdots \leq i_k
\leq |X|$ such that the following are true
\begin{enumerate}  
	\item Between $\lambda_{i_j}$ and $\lambda_{i_{j}+1}$, there is a non-trivial
	gap, \ie for $j \in \{i_1, \ldots, i_k \}$, $\frac{\lambda_{j +
	1}}{\lambda_{j}} \leq \gamma  < 1$
	\item Let $i_0 = 1$, we refer to $S_j = \{i_{j-1} + 1, \ldots, i_j\}$ as the
	$j\th$ block (of eigenvalues and eigenvectors). Then the size of each block,
	$|S_j| \leq N$
	\item The eigenvalue $\lambda_{i_k}$ is not too small (with respect to the
	gap at the end of each block), $\lambda_{i_k} \geq \gamma^{c}$
\end{enumerate}
\end{definition}

In general, the parameter $\gamma$ will depend on $n$ and we require that
$\gamma \leq 1 - 1/\poly(n)$ in order to separate eigenvectors from the various
blocks. One would expect $N$ to have dependence on both $n$ and $k$
and $c$ to have some dependence on $k$.  As an example, we note that the
spectrum corresponding to the Markov chain discussed in
Section~\ref{sec:unif-boolean} is indeed discrete with the following parameters:
$k$ can be any integer, $N = n^k$, $\gamma = 1 - (1/n)$ and $c = O(k)$. 

In order to extract eigenvectors of $P$, we start with a collection of functions
which have significant \emph{Fourier mass} on the top eigenvectors. For an
eigenvector $\nu$, it's \emph{Fourier coefficient} in any function $g : X
\rightarrow  \reals$ is simply $\ip{g}{\nu}$. Condition 2 in
Definition~\ref{defn:useful-basis} implicitly requires that the inner product
$\langle g, \nu \rangle$ be large for $\nu$ with a large eigenvalue for some $g$
in the set. In addition, since eigenvalues corresponding to different
eigenvectors may be equal or close together, we require a set of functions where
the matrix corresponding to the \emph{Fourier coefficients} of such eigenvectors is
well-conditioned.  Formally, we define the notion of a \emph{useful basis} of
functions with respect to a transition matrix $P$ which has an $(N, k, \gamma,
c)$-discrete spectrum.

\begin{definition}[Useful Basis] \label{defn:useful-basis} Let $\G$ be a
collection of functions from $X \rightarrow \reals$. We say that $\G$ is
$\alpha$-useful for an $(N, k, \gamma, c)$-discrete $P$ if the following hold:
\begin{enumerate}
\item For every $g \in \G$, $\Vert g \Vert_{\infty} \leq 1$
\item Let $i_0 = 0$, then for any $1 \leq j \leq k$, if $N_j = i_j - i_{j-1}$
(the size of the $j\th$ block), there exist $N_j$ functions $g_1, \ldots,
g_{N_j} \in \G$, such that the $N_j \times N_j$ matrix $A$ defined by
$a_{m,l} = \ip{g_m}{\nu_l}$, where $m \in \{1, \ldots, N_j\}$ and $l \in
\{i_{j-1} + 1, \ldots, i_j \}$, has smallest singular value at least $1/\alpha$.
Alternatively, the operator norm of $A^{-1}$, $\operatornorm{A^{-1}}$ is at most
$\alpha$.
\end{enumerate}
\end{definition}

The parameter $\alpha$ will have dependence on $N$---a polynomial dependence on
$N$ would result in efficient algorithms. In general, it is not known which
Markov chains admit a useful basis that has a succinct representation. In the
case of the uniform distribution, clearly the collection of parity functions
already is such a \emph{useful basis}. However, we observe that there are other
useful bases as well. For example if for some $k$, one wished to extract all
eigenvectors with eigenvalues at least $1 - k/n$ (parities of size at most $k$),
one can start with the collection of functions that is disjunctions (or
conjunctions) on at most $k$ variables. Note that in this case, there is no
contribution from eigenvectors with low eigenvalues (\ie noise) in the basis
functions. However, one would not expect to find such a \emph{useful basis}
without any contributions from eigenvectors with low eigenvalues when the
stationary distribution is not product. 

We now show how functions from a \emph{useful basis} for a transition matrix with a
discrete spectrum can be used to extract eigenvectors. First by applying powers
of $P$ to some function $g$, the contributions of eigenvectors in different
blocks can be separated. However, to separate eigenvectors within a block we
require an \emph{incoherence condition} among the various $g_m$s (which is the
second condition in Definition~\ref{defn:useful-basis}). We first show that the
eigenvectors $\nu$ can be approximately represented in the following form:
\[ \nu \approx \sum_{t, m} \beta_{t, m} P^t g_m, \] 
where $m$ indexes the functions in $\G$.

\begin{theorem} \label{thm:spectrum} Let $P$ be a transition matrix with an $(N,
k, \gamma, c)$
discrete spectrum and let  $\G$ be an $\alpha$-useful basis for $P$. Then for any
$\epsilon > 0$, there exists $\taumax$ and $B$ such that every eigenvector
$\nu_\ell$ with $\ell \leq i_k$ can be expressed as:
\[ \nu_{\ell} = \sum_{t,m} \beta^{\ell}_{t,m} P^t g_m + \eta_i \]
where $\twonorm{\eta_i} \leq \epsilon$, $t \leq \taumax$ and $\sum_{t, m}
|\beta_{t,m}| \leq B$. Furthermore, 
\begin{align*}
B &= (2 \alpha N k )^{\Theta((1 + c)^{k+1})} \epsilon^{-(1 + c)^k}  \\
\tau_{max} &= O \left( k ( 1 + c)^{k-1} (\log (N) + \log(k) + \log(\alpha) +
+ \frac{1}{\log(1/\gamma)} + \log(\frac{1}{\epsilon})) \right)
\end{align*}
\end{theorem}

The proof of the above theorem is somewhat delicate and is provided in
Appendix~\ref{app:eigen1}. Notice that the bounds on $B$ and $\tau$ have a
relatively mild (polynomial) dependence on most parameters except $k$ and $c$.
Thus, when $c$ and $k$ are relatively small, for example both of them
constant, both $B$ and $\tau$ are bounded by polynomials in the other
parameters. Also, $N$ may be somewhat large, in the case of the uniform
distribution $N = \Theta(n^k)$---though this is still polynomial if $k$ is
constant.

We can now use the above Theorem to devise a simple learning algorithm with
respect to stationary distribution of the Markov chain. In fact, the learning
algorithm does not even need to explicitly estimate the values of
$\beta^{\ell}_{t, m}$ in the statement of Theorem~\ref{thm:spectrum}---the
result shows that any linear combination of the eigenvectors can also be
represented as a linear combination of the collection of functions $\{ P^t g_m
\}_{t \leq \taumax, g_m \in \G}$.  Thus, we can treat this collection as
``features'' and simply perform linear regression (either $L_1$ or $L_2$) as 
part of the learning algorithm. The algorithm is given in
Figure~\ref{fig:learning-algorithm}. The key idea is to show that $P^tg_m(x)$
can be approximately computed for any $x \in X$ with blackbox access to $g_m$
and the one-step oracle $OS(\cdot)$. This is because $P^t g_m(x) = \E_{y \sim
P^t(x, \cdot)} [g(y)]$, where $P^t(x, \cdot)$ is the distribution over $X$
obtained by starting from $x$ and taking $t$ steps of the Markov chain. The
functions $\phi_{t, m}$ in the algorithm are computing approximations to $P^t
g_m$ and then using them as features for learning.  Formally, we can prove the
following theorem.

\begin{figure}
\fbox{
\begin{minipage} {0.95 \textwidth}

{\bf Inputs}: $\taumax, W, T$, blackbox access to $g \in \G$ and $\OS(\cdot)$,
labeled examples $\langle (x_i, y_i) \rangle_{i = 1}^s$ \medskip \\

{\bf Preprocessing}: 
For each $t \leq \taumax$ and $m$ such that $g_m \in \G$
\begin{itemize}
\item For each $i = 1, \ldots, s$, let 
\begin{align}
\phi_{t, m}(x_i) = \frac{1}{T} \sum_{j = 1}^T g_m(\OS^t_j(x_i)),
\label{eqn:phi-define}
\end{align}
where $\OS^t_j(x_i)$ denotes the point obtained by an independent forward
simulation of the Markov chain starting at $x_i$ for $t$ steps, for each $j$.
\end{itemize} \medskip 

{\bf Linear Program}:
Solve the following linear program:
\begin{align*}
\mbox{minimize} &\sum_{i = 1}^s  \left|\sum_{t \leq \taumax, g_m \in \G} w_{t, m}
\phi_{t, m}(x_i) - y_i\right| \\
\mbox{subject to} & \sum_{t \leq \taumax, g_m \in \G} |w_{t, m}| \leq W
\end{align*} \medskip 

{\bf Output Hypothesis}:
\begin{itemize}
\item Let $h(x) = \displaystyle\sum_{t \leq \taumax, g_m \in \G} w_{t, m} \phi_{t, m}(x)$, where
$\phi_{t,m}(x)$ are defined as in step~(\ref{eqn:phi-define}) above.
\item Let $\theta \in [-1, 1]$ be chosen uniformly at random and output
$\sign(h(x) - \theta)$ as prediction
\end{itemize}

\end{minipage}
}
\caption{Agnostic Learning with respect to MRF distributions
\label{fig:learning-algorithm}}
\end{figure}

\begin{theorem} \label{thm:agnostic-learning} Let $M = \langle X, P \rangle$ be
a Markov chain and let $\lambda_1 \geq \lambda_2 \geq \cdots \geq 0$ denote the
eigenvalues of $P$ and $\nu_{\ell}$ the eigenvector corresponding to
$\lambda_{\ell}$. Let $\pi$ be the stationary distribution of $P$. Let $\F$ be
a class of boolean functions. Suppose for some $\epsilon > 0$, there exists
$\ell^*(\epsilon)$ such that for every $f \in \F$, 
\[ \sum_{\ell > \ell^*} \ip{f}{\nu_{\ell}}^2 \leq \tfrac{\epsilon^2}{4}, \]
\ie every $f$ can be approximated (up to $\epsilon^2/4$) by the top $\ell^*$
eigenvectors of $P$. Suppose $P$ has a $(N, k, \gamma, c)$-discrete spectrum as
defined in Definition~\ref{defn:discrete-spectrum}, with $i_k \geq \ell^*$ and
that $\G$ is an $\alpha$-useful basis for $P$. Then, there exists a learning
algorithm that with blackbox access to functions $g \in \G$, the one-step
oracle $\OS(\cdot)$ for Markov chain $M$, and access to random examples $(x,
L(x))$ where $x \sim \pi$ and $L$ is an arbitrary labeling function,
agnostically learns $\F$, up to error $\epsilon$.  

Furthermore, the running time, sample complexity and the time required to
evaluate the output hypothesis are bounded by a polynomial in $(Nk)^{(1 +
c)^{k+1}}, \epsilon^{-(1 + c)^k}, |\G|, n$. In particular, if $\epsilon$ is
a constant, $c$ and $k$ depend only on $\epsilon$ (and not on $n$), and $N
\leq n^{\zeta(k)}$, where $\zeta$ may be an arbitrary function, the algorithm
runs in polynomial time.
\end{theorem}

We give the proof this theorem in Appendix~\ref{app:eigen2}; the proof uses the
$L_1$-regression technique of \cite{KKMS:2005}. We comment that the learning
algorithm (Fig.~\ref{fig:learning-algorithm}) is a generalization of the
low-degree algorithm of \citet{LMN:1993}.  Also, when applied to the Markov
chain corresponding to the uniform distribution over $\moo^n$, this algorithm
works whenever the low-degree algorithm does (albeit with slightly worse
bounds). As an example, we consider the algorithm of \citet{KOS:2002} to learn
arbitrary functions of halfspaces. As a main ingredient of their work, they
showed that halfspaces can be approximated by the first $O(1/\epsilon^4)$ levels
of the Fourier spectrum. The running time of our learning algorithm run with a
useful basis consisting of parities, or conjunctions of size $O(1/\epsilon^4)$
is polynomial (for constant $\epsilon$).

\subsection{Noise Sensitivity Analysis}

\label{sec:noise-sensitivity}

In light of Theorem~\ref{thm:agnostic-learning}, one can ask which function
classes are well-approximated by top eigenvectors and for which MRFs. A generic
answer is functions that are ``noise-stable'' with respect to the
underlying Gibbs Markov chain. Below, we generalize the definition of noise
sensitivity in the case of product distributions to apply under MRF distributions.
In words, the noise sensitivity (with parameter $t$) of a boolean function $f$
is the probability that $f(x)$ and $f(y)$ are different, where $x \sim \pi$ is
drawn from the stationary distribution and $y$ is obtained by taking $t$ steps
of the Markov chain starting at $x$. 

\begin{definition} Let $x \sim \pi$ from the stationary distribution of $P$ and
$y \sim  P^t(x, \cdot)$, the distribution obtained by taking $t$ steps of the
Gibbs MC starting at $x$. For a boolean function $f : X \rightarrow \{-1, 1\}$,
define its noise sensitivity with respect to parameter $t$ and the transition
matrix $P$ of the Gibbs MC as
\[ \NS_t(f) = \Pr_{x \sim \pi, y \sim P^t(x, \cdot)}[f(x) \neq f(y)]. \]
\end{definition}


One can derive an alternative form for the noise sensitivity as follows. Let
$\lambda_1 \geq \lambda_2  \geq \cdots 0$ denote the eigenvalues of $P$ and
$\nu_1, \nu_2, \ldots$ the corresponding eigenvectors. Let $\hat{f}_{\ell} =
\ip{f}{\nu_{\ell}}$. Then,
\begin{align}
\NS_t(f) &= \Pr_{x \sim \pi, y\sim P^t(x, \cdot)}[f(x) \neq f(y)]
\nonumber \\
&= \frac{1}{2}\E_{x \sim \pi, y \sim P^t(x, \cdot)} [1 - f(x)f(y)] \nonumber \\
&= \frac{1}{2} - \frac{1}{2} \ip{f}{P^tf} \nonumber \\
&= \frac{1}{2} - \frac{1}{2} \sum_{\ell} \lambda_{\ell}^t \hat{f}^2_{\ell}
\label{eqn:noise-sensitivity}
\end{align}

The notion of noise-sensitivity has been fruitfully used in the theory of learning
under the uniform distribution (see for example~\cite{KOS:2002}). The main idea
is that functions that have low noise sensitivity have most of their mass
concentrated on ``lower order Fourier coefficients'', \ie eigenvectors with
large eigenvalues. We show that this idea can be easily generalized in the
context of MRF distributions. The proof of the following theorem is provided in
Appendix~\ref{app:eigen3}.

\begin{theorem} \label{thm:noise-sensitivity} Let $P$ be the transition matrix
of the Gibbs MC of an MRF and let $f : X \rightarrow \{-1, 1\}$ be a boolean
function. Let $\ell^*$ be the largest index such that $\lambda_{\ell^*} > \rho$,
then: 
\[ \sum_{\ell > \ell^*} \hat{f}_{\ell}^2 \leq \frac{e}{e - 1}
\NS_{-\frac{1}{\ln{\rho}}}(f) \]
\end{theorem}

Thus, it is of interest to study which function classes have low
noise-sensitivity with respect to certain MRFs. As an example, we consider the
Ising model on graphs with bounded degrees; the Gibbs MC in this case is the
Glauber dynamics. We show that the class of halfspaces have low noise
sensitivity with respect to this family of MRFs. In particular, the noise
sensitivity with parameter $t$, only depends on $(t/n)$.

\begin{proposition} \label{prop:one}
For every $\Delta \geq 0$, there exists $\beta(\Delta) > 0$ such that the
following holds: For every graph $G$ with maximum degree $\Delta$, the Ising
model with $\beta < \beta(\Delta)$ and any function of the form $f =
\sign(\sum_{i=1}^n w_i x_i)$, it holds that $\NS_{t}(f) \leq \exp(-\delta (t/n))$,
for some constant $\delta$ that depends only on $\Delta$.
\end{proposition}

The proof of the above proposition follows from
Lemma~\ref{lemma:new_one} in Appendix~\ref{app:eigen4}. As a corollary
we get.

\begin{corollary} Let $P$ be the transition matrix of the Gibbs MC of an Ising
model with bounded degree $\Delta$. Suppose that for some $\epsilon > 0$, $P$
has an $(N, k, \gamma, c)$-discrete spectrum such that $k$ depends only on
$\epsilon$ and $\Delta$, $\lambda_{i_k + 1} < \exp(-\tfrac{\delta}{1} \cdot \tfrac{1}{\ln(4/\epsilon^2)})$
(where $\delta$ is as in Proposition~\ref{prop:one}), $\gamma = 1 -
1/\poly(n)$, $N$ is $\poly(n)$ and $c$ a constant, for constant $\epsilon,
\Delta$. Furthermore, suppose that $P$ admits an $\alpha$-useful basis with
$\alpha= \poly(n, 1/\epsilon)$, for the parameters $(N, k, \gamma, c)$ as
above. Then the class of halfspaces $\{\sign(\sum_{i} w_i x_i) \}$, is
agnostically learnable with respect to the stationary distribution $\pi$ of $P$
up to error $\epsilon$.  \end{corollary}
\begin{proof}
Let $t = \tfrac{n}{\delta} \ln(4/\epsilon^2)$, where $\delta$ is from
Proposition~\ref{prop:one}. Thus, $\NS_t(f) \leq \epsilon^2/4$. Let $\rho =
\exp(-1/t)$ (as in Theorem~\ref{thm:noise-sensitivity}); by the assumption on
$P$, $P$ admits an $(N, k, \gamma, c)$-distribution where $k$ depends only on
$\epsilon, \Delta$, such that $\lambda_{i_k + 1} < \rho$. 

Now, the algorithm in Figure~\ref{fig:learning-algorithm} together with the
parameter settings from Theorems~\ref{thm:spectrum},
\ref{thm:agnostic-learning} and \ref{thm:noise-sensitivity} give the desired
result.
\end{proof}

\subsection{Discussion}

In this section, we proposed that approximation using eigenvectors of the
transition matrix of an appropriate Markov chain may be better than just
polynomial approximation, when learning with respect to distributions defined by
Markov random fields (not product). We checked this for a few different Ising
models to approximate the majority function.
Since the computations required are fairly intensive, we could only do this for
relatively small models. However, we point that the methods proposed in this
paper are highly-parallelizable and not beyond the reach of large computing
systems. Thus, it may be of interest to run methods proposed here on larger
datasets and real-world data. \medskip 

\noindent {\bf Approximation of Majority}: 
We look at three different graphs: a cycle of length 11, the complete graph on
11 nodes and an Erd\H{o}s-R\'{e}nyi random graph with $n = 11$ and $p = 0.3$.
We looked at the Ising model on these graphs with various different values of
$\beta$. In each case, we looked at degree-$k$ polynomial approximations for $k
= 2, 4$ and also with using top $n^k$ eigenvectors of the majority function. We
see that the approximation using eigenvectors is consistently better, except
possibly for very low values of $\beta$, where polynomial approximations are
also quite good. The values reported in the table are squared error for the
approximation.

\begin{table}
\begin{center}
$
\begin{array}{ccc}
\mbox{
\begin{tabular}{|c|c|c|c|} \hline
$\beta$ & Degree & Poly & Eigen \\ \hline
0.02 & 2 & 0.3321 & 0.3550 \\
	 & 4 & 0.2084 & 0.1645 \\ \hline
0.05 & 2 & 0.3184 & 0.2322 \\
     & 4 & 0.1937 & 0.1648 \\ \hline
0.1  & 2 & 0.2238 & 0.1417 \\ 
     & 4 & 0.1199 & 0.0687 \\ \hline
0.2  & 2 & 0.1468 & 0.0018 \\ 
     & 4 & 0.0034 & 0.0013 \\ \hline
\end{tabular}
}
 & 
\mbox{
\begin{tabular}{|c|c|c|c|} \hline
$\beta$ & Degree & Poly & Eigen \\ \hline
0.1 & 2 & 0.3330 & 0.3401 \\
	 & 4 & 0.2092 & 0.1606 \\ \hline
0.2 & 2 & 0.3307 & 0.2229 \\
     & 4 & 0.2052 & 0.1538 \\ \hline
0.5  & 2 & 0.3113 & 0.1918 \\ 
     & 4 & 0.1676 & 0.0715 \\ \hline
1.0  & 2 & 0.1857 & 0.0466 \\ 
     & 4 & 0.0344 & 0.0253 \\ \hline
\end{tabular}
}
& 
\mbox{
\begin{tabular}{|c|c|c|c|} \hline
$\beta$ & Degree & Poly & Eigen \\ \hline
0.05 & 2 & 0.3327 & 0.3404 \\
	 & 4 & 0.2089 & 0.2172 \\ \hline
0.1 & 2 & 0.3283 & 0.2240 \\
     & 4 & 0.2034 & 0.1515 \\ \hline
0.2  & 2 & 0.3017 & 0.1897 \\ 
     & 4 & 0.1757 & 0.1254 \\ \hline
0.5  & 2 & 0.0690 & 0.0326 \\ 
     & 4 & 0.0262 & 0.0108 \\ \hline
\end{tabular}
} \\
\mbox{(a) $K_{11}$} & \mbox{(b) $C_{11}$} & \mbox{(c) $G(11, 0.3)$} 
\end{array}
$
\end{center}
\caption{Approximation of the majority function using polynomials and
eigenvectors for different Ising models}
\end{table}

\section{Learning Juntas}
\label{sec:juntas}
In this section, we consider the problem of learning the class of $k$-juntas.
Suppose $X = \scA^n$ is the instance space. A $k$-\emph{junta} is a boolean
function that depends on only $k$ out of the $n$ possible co-ordinates of $x \in
X$. In this section, we consider the model in which we receive labeled examples
from a random walk of a Markov chain (see
Section~\ref{sec:model}.2).\footnote{In the model where labeled examples are
received from the only from stationary distribution, it seems unlikely that any
learning algorithm can benefit from access to the $\OS(\cdot)$ oracle. The
problem of learning juntas in time $n^{o(k)}$ is a long-standing open problem
even when the distribution is uniform over the Boolean cube, where the
$\OS(\cdot)$ oracle can easily be simulated by the learner itself.} In this case
the learning algorithm can identify the $k$ relevant variables by keeping track
of which variables caused the function to change its value. 

For a subset, $S \subseteq [n]$ of the variables and a function $b_S : S
\rightarrow \scA$, let $x_S = b_S$ denote the event, $\bigwedge_{i \in S} x_i =
b_S(x_i)$, \ie it fixes the assignment on the variables in $S$ as given by the
function $b_S$. A set $S$ is the \emph{junta} of function $f$, if the variables
in $S$ completely determine the value of $f$. In this case, for $b_S : S
\rightarrow \scA$, every $x$ satisfying $x_S = b_S$ has the same value $f(x)$
and by slight abuse of notation we denote this common value by $f(b_S)$.

Figure~\ref{alg:learn-juntas} describes the simple algorithm for learning
juntas.  Theorem~\ref{thm:junta} gives conditions under which
Algorithm~\ref{alg:learn-juntas} is guaranteed to succeed. Later, we show that
the Ising model and graph coloring satisfy these conditions.
%
%
\begin{figure}[!t]
\begin{center}
\fbox{
\begin{minipage}{0.95 \textwidth}
{\bf Inputs}: Access to labeled examples $(x, f(x))$ from Markov Chain $M$ \medskip 

{\bf Identifying Relevant Variables}
\begin{enumerate}
\item ${\mathcal J} = \emptyset$
\item Consider a random walk, $\langle (x^1, f(x^1)), \ldots, (x^T,
f(x^T) \rangle$. 
\item For every, $i$, such that $f(x^i) \neq f(x^{i+1})$, if $j$ is the variable
such that $x^i_j \neq x^{i+1}_j$, add $j$ to ${\mathcal J}$.
\end{enumerate}
 
{\bf Learning $f$}
\begin{enumerate}
\item Consider each of the $|\scA|^{|{\mathcal J}|}$ possible assignments
$b_{{\mathcal J}} \rightarrow \scA$. We will construct a truth table for a
function $h : \scA^{\junta} \rightarrow \yy$. 
\item For a fixed $b_{\junta}$, let $h(b_{\junta})$ be the plurality label among
the $x^i$ in the random walk above for which $x^i_j = b_{\junta}(j)$ for all $j
\in \junta$. 
\end{enumerate} \medskip

{\bf Output}: Hypothesis $h$
\end{minipage}
}
\end{center}
\caption{Algorithm: Exact Learning $k$-juntas \label{alg:learn-juntas}}
\end{figure}

\begin{theorem} 
\label{thm:junta} Let $X = \scA^n$ and let $M = \langle X, P \rangle$ be a
time-reversible rapidly mixing MC. Let $\pi$ denote the stationary distribution
of $M$ and $\tau_M$ its mixing time. Furthermore, suppose that $M$ has
\emph{single-site} dynamics, \ie $P(x, x^\prime) = 0$ if $d_H(x,x^\prime) > 1$
and that the following conditions hold:  \smallskip \\
$\mbox{~~~}$(i)~For any $S \subseteq [n]$, $b_S : S \rightarrow \scA$ either
$\pi(x_S = b_S) = 0$ or $\pi(x_S = b_S) \geq 1/(c|\scA|)^{|S|}$, where $c$ is a
constant. \smallskip \\
$\mbox{~~~}$(ii)~For any $x, x^\prime$ such that $\pi(x) \neq 0$, $\pi(x^\prime)
\neq 0$ and $d_H(x, x^\prime) = 1$, $P(x, x^\prime) \geq \beta$. \smallskip \\
Then Algorithm~\ref{alg:learn-juntas} \emph{exactly} learns the class of
$k$-junta functions with probability at least $1 - \delta$ and the running time
is polynomial in $n, |\scA|^k, \tau_M, 1/\beta, \log(1/ \delta)$. 
\end{theorem}

\begin{proof}
Let $f$ be the unknown target $k$-junta function. Let $S$ be the set of
variables that influence $f$, $|S| \leq k$. The set $S$ is called the
\emph{junta} for $f$.  Note that a variable $i$ is in the junta for $f$, if and
only if there exist $x, x^\prime \in \scA^n$ such that $\pi(x) \neq 0$,
$\pi(x^\prime) \neq 0$, $x$, $x^\prime$ differ only at co-ordinate $i$ and $f(x)
\neq f(x^\prime)$. Otherwise, $i$ can have no influence in determining the value
of $f$ (under the distribution $\pi$).

We claim that Algorithm~\ref{alg:learn-juntas} identifies every variable in the
junta $S$ of $f$. Let $b_S : S \rightarrow \scA$, be any assignment of values to
variables in $S$. Since $S$ is the \emph{junta} for $f$, any $x \in X$ that
satisfies $x_i = b_S(i)$ for all $i \in S$, has the same value $f(x)$. By slight
abuse of notation, we denote this common value by $f(b_S)$.

The fact that $i \in S$ implies that there exist assignments, $b^1_{S}$,
$b^2_{S}$, such that $b^1_S(i) \neq b^2_S(i)$, $\forall j \in S$, such that $j
\neq i$, $b^1_S(j) = b^2_S(j)$ and which satisfy the following: $\pi(x_S =
b^1_S) \neq 0$, $\pi(x_S, b^2_S) \neq 0$. Consider the following event: $x$ is
drawn from $\pi$, $x^\prime$ is the state after exactly one transition, $x$
satisfies the event $x_S = b^1_S$ and $x^\prime$ satisfies the event $x^\prime_S
= b^2_S$. By our assumptions, the probability of this event is at least
$\beta/(c|\scA|)^{|S|}$.  Let $\alpha = \beta/(c|\scA|)^{|S|}$. Then, if we draw
$x$ from the distribution $P^t(x_0, \cdot)$ for $t = \tau_M \ln(2/\alpha)$,
instead of the \emph{true} stationary distribution $\pi$, the probability of the
above event is still at least $\alpha/2$. This is because when $t = \tau_M
\ln(2/\alpha)$, the $\dtv{P^t(x_0, \cdot) - \pi} \leq \alpha/2$.  Thus, by
observing a long enough random walk, \ie one with $2 \tau_M
\ln(1/\alpha)\log(k/\delta)/\alpha$ transitions, except with probability
$\delta/k$, the variable $i$ will be identified as a member of the junta. Since
there are at most $k$ such variables, by a union bound all of $S$ will be
identified. Once the set $S$ has been identified, the unknown function can be
learned \emph{exactly} by observing an example of each possible assignments to
the variables in $S$. The above argument shows that all such assignments with
non-zero measure under $\pi$ already exist in the observed random walk.
\end{proof}

\begin{remark}
We observe that the condition that the MC be \emph{rapidly mixing} alone is
sufficient to identify at least one variable of the junta. However, unlike in
the case of learning from i.i.d. examples, in this learning model, identifying
one variable of the junta is not equivalent to learning the unknown junta
function. In fact, it is quite easy to construct rapidly mixing Markov chains
where the \emph{influence} of some variables on the target function can be
hidden, by making sure that the transitions that cause the function to change
value happen only on a subset of the variables of the junta.
\end{remark}

We now show that the Ising model and graph coloring satisfy the conditions of
Theorem~\ref{thm:junta} as long as the underlying graphs have constant degree.  \medskip

\noindent{\bf Ising Model}: Recall that the state space is $X = \moo^n$. Let
$\beta(\Delta)$ be the inverse critical temperature, which is a constant
independent of $n$ as long as $\Delta$, the maximal degree, is constant. Let $S
\subseteq [n]$ and let $b^1_S  : S \rightarrow \moo$ and $b^2_S : S \rightarrow
\moo$ be two distinct assignments to variables in $S$. Let $\sigma^1, \sigma^2$
be two configurations of the Ising system such that for all $i \in S$,
$\sigma^1_i = b^1_S(i)$, $\sigma^2_i = b^2_S(i)$ and for $i \not\in S$,
$\sigma^1_i = \sigma^2_i$. Let $d^1 = \sum_{(i, j) \in E: \sigma^1_i \neq
\sigma^1_j} \beta_{ij}$ and $d^2 = \sum_{(i, j) \in E: \sigma^2_i \neq
\sigma^2_j} \beta_{ij}$. Then, since the maximum degree of the graph $\Delta$ is
constant and each $\beta_{ij}$ is also bounded by some constant, $|d^1 - d^2|
\leq c |S| \Delta$. Then, by definition (see Section~\ref{sec:prelim}), $\exp(-c
\beta \Delta |S|) \leq \pi(\sigma^1)/\pi(\sigma^2) \leq \exp(c \beta \Delta
|S|)$. By summing over possible pairs $\sigma^1, \sigma^2$ that satisfy the
constraints, we have $\exp(-\beta \Delta |S|) \leq \pi(x_S = b^1_S)/\pi(x_S =
b^2_S) \leq \exp(\beta \Delta |S|)$. But, since there are only $2^{|S|}$
possible assignments of variables in $S$, the first assumption of
Theorem~\ref{thm:junta} follows immediately. The second assumption follows from
the definition of the transition rate matrix, \ie each non-zero entry in the
transition rate matrix is at least $\exp(-\beta \Delta)/2n$.  \medskip

\noindent{\bf Graph Coloring}: Let $q$ be the number of colors. The state
space is $[q]^n$ and \emph{invalid} colorings have $0$ mass under the
stationary distribution. We assume that $q \geq 3 \Delta$, where $\Delta$ is the
maximum degree in the graph. This is also the assumption that ensures rapid
mixing. Let $S \subseteq [n]$ be an subset of nodes. Let $C^1_S$ and $C^2_S$ be
two assignments of colors to the nodes in $S$. Let $D_1$ and $D_2$ be the set
of valid colorings such that for each $x \in D_1$, $i \in S$, $x_i = C^1_S(i)$
and for each $x \in D_2$, $i \in S$, $x_i = C^2_S(i)$. We define a map from
$D_1$ to $D_2$ as follows: 
\begin{enumerate}
\item Starting from $x \in D_1$, first for all $i \in S$, set $x_i = C^2_S(i)$.
This may in fact result in an \emph{invalid} coloring.
\item The invalid coloring is switched to a valid coloring by only modifying
neighbors of nodes in $S$. The condition that $q \geq 3 \Delta$ ensures that
this can always be done.
\end{enumerate}

The above map has the following properties. Let $N(S) = \{ j ~|~ (i, j) \in E, i
\in S\}$. Then, the nodes that are not in $S \cup N(S)$ do not change the
color. Thus, even though the map may be a many to one map, at most $q^{|S| +
|N(S)|}$ elements in $D_1$ may be mapped to a single element in $D_2$. Note that
$|S| + |N(S)| \leq (\Delta+ 1) |S|$. Thus, we have $\pi(D_1)/\pi(D_2) =
|D_1|/|D_2| \leq q^{(\Delta+1)|S|}$. This implies the first condition of
Theorem~\ref{thm:junta}. The second condition follows from the definition of the
transition matrix, each non-zero entry is at least $1/(2qn)$.

\bibliography{all,all-refs}
\bibliographystyle{plainnat}

\newpage 

\appendix

\section{Proofs from Section~\ref{sec:eigen}}

\label{app:eigen}
\subsection{Proof of Theorem~\ref{thm:spectrum}}
\label{app:eigen1}

\begin{proof}
We divide the spectrum of $P$ into blocks. Let $k$ and $i_1, \ldots, i_k$ be as
in Definition~\ref{defn:discrete-spectrum}; furthermore define $i_0 = 0$ for
notational convenience. For $j = 1, \ldots, k$, let $S_j = \{ i_{j- 1} + 1,
\ldots, i_j \}$. Throughout this proof we use the letter $\ell$ to index
eigenvectors of $P$---so $\nu_{\ell}$ is an eigenvector with eigenvalue
$\lambda_{\ell}$. We want to find $\beta^{\ell}_{t,m}$ in order to
(approximately) represent the eigenvector $\nu_{\ell}$ as
\begin{align}
	\nu_\ell &= \sum_{t, m} \beta^{\ell}_{t,m} P^t g_m + \eta_\ell
	\label{eqn:eigvec-rep}
	\intertext{Also, we use the notation,}
	\bar{\nu}_{\ell} &= \sum_{t, m } \beta^{\ell}_{t,m} P^t g_m,
\end{align}
We will show that such representations exist block by block. To begin
define 
\begin{align}
\epsilon_1 &= \left( \frac{\epsilon}{(2 \alpha N)^{\frac{1+c}{c}}
(Nk)^{\frac{1}{2c}}} \right)^{(1 + c)^{k-1}} \label{eqn:epsilon1-define}
\intertext{ and define $\epsilon_j$ according to the following recurrence,}
\epsilon_j &= 2 \alpha N (Nk)^{\frac{1}{2(1 + c)}} \epsilon_{j - 1}^{\frac{1}{1 + c}}
\label{eqn:epsilon-recurrence}
\intertext{It is an easy calculation to verify that the solution for
$\epsilon_j$ is given by}
\epsilon_j &= \left(2 \alpha N (Nk)^{\frac{1}{2(1 + c)}}\right)^{\frac{1 +
c}{c} \left(1 - \frac{1}{(1 + c)^{j-1}} \right)}
\epsilon_1^{\frac{1}{(1 + c)^{j-1}}} \label{eqn:epsilon-value}
\intertext{Also, define}
B_1 &= (N\alpha)^{c + 1} \epsilon_1^{-c} \label{eqn:B1-define} 
\intertext{and let $B_j$ be defined according the following recurrence:}
B_j &= 2  \alpha N (Nk)^{\frac{1}{2(1 + c)}} (\epsilon_{j-1})^{-\frac{c}{1 +
c}} B_{j-1} \label{eqn:B-recurrence}
\intertext{It is an easy calculation to verify that the solution for $B_j$ is
given by}
B_j &= \left(2 N \alpha (Nk)^{\frac{1}{2(1 + c)}} \right)^{j-1} \cdot \left(
\prod_{j^\prime = 1}^{j-1} \epsilon_{j^\prime} \right)^{-\frac{c}{1 + c}}
B_1 \label{eqn:B-value}
\end{align}

It can be verified that $\epsilon_j$ and $B_j$ are increasing as a function of
$j$ as long as all $\epsilon_j$ remain smaller than $1$ (which can be verified by
checking that $\epsilon_k < 1$).  We show by induction on $j$ that for any $\ell
\in S_j$, $\sum_{t, m} |\beta^{\ell}_{t,m}| \leq B_j$ and $\twonorm{\eta_{\ell}}
\leq \epsilon_j$ (recall that the norm here is with respect to the distribution
$\pi$). 

Consider some $j$ and suppose that $|S_j| = N_j$. Denote by $S_{<j} =
\displaystyle\bigcup_{j^\prime < j} S_{j^\prime}$, all the indices that precede
those in $S_j$ and $S_{>j} = \{\ell^\prime ~|~ \ell^\prime > i_j\}$. According
to Definition~\ref{defn:useful-basis}, there exist $g_1, \ldots, g_{N_j} \in
\G$, such that if $A$ is the $N_j \times N_j$ matrix given by $a_{m, \ell} =
\ip{g_m}{\nu_\ell}$ for $\ell \in S_j$ and $1 \leq m \leq N_j$, then
$\operatornorm{A^{-1}} \leq \alpha$.  Let $\bar{a}_{\ell, m}$ denote the
element in position $(l, m)$ in $A^{-1}$ and let $\G_j = \{g_1, \ldots, g_{N_j}\}$
be these specific $N_j$ functions in $\G$.  Also, observe that by
Definition~\ref{defn:discrete-spectrum}, $N_j \leq N$.

Let $g_m \in \G_j$ and for any $\ell^\prime$, let $a_{m, \ell^\prime} =
\ip{g_m}{\nu_{\ell^\prime}}$. Then, define
\begin{align}
	\tilde{g}_m &= g_m - \sum_{\ell^\prime \in S_{<j}} a_{m, \ell^\prime}
	\bar{\nu}_{\ell^\prime} \label{eqn:g-modify}
\intertext{Thus, $\tilde{g}_m$ is obtained from $g_m$ by (approximately)
removing contributions of eigenvectors corresponding to blocks that precede the
$j\th$ block. Thus, we may write $\tilde{g}_m$ as follows:}
\tilde{g}_m &= \sum_{\ell \in S_{j}} a_{m, \ell} \nu_{\ell} + \sum_{\ell^\prime
\in S_{<j}} a_{m, \ell^\prime} (\nu_{\ell^\prime} - \bar{\nu}_{\ell^\prime}) +
\sum_{\ell^\prime \in S_{>j}} a_{m, \ell^\prime} \nu_{\ell^\prime}
\nonumber \\
&= \sum_{\ell \in S_j} a_{m, \ell} \nu_{\ell} + \sum_{\ell^\prime \in S_{<j}}
a_{m, \ell^\prime} \eta_{\ell^\prime} + \sum_{\ell^\prime \in S_{>j}}
a_{m, \ell^\prime} \nu_{\ell^\prime} \nonumber
\intertext{To further simplify the above equation, define $v^<_m =
\sum_{l^\prime \in S_{<j}} a_{m, \ell^\prime} \eta_{\ell^\prime}$ and $v^>_m =
\sum_{l^\prime \in S_{>j}} a_{m, \ell^\prime} \nu_{\ell^\prime}$. Then, we have}
\tilde{g}_m &= \sum_{\ell \in S_j} a_{\ell, m} \nu_{\ell} + v^<_m + v^>_m \label{eqn:rep}
\end{align}
In the case of $v^<_m$, a crude bound can be established on its norm
$\twonorm{v^<_m}$ as follows: for any $\ell^\prime \in S_{<j}$,
$\twonorm{\eta_{\ell^\prime}} \leq \epsilon_{j-1}$ (induction hypothesis).
Using the facts that $\sum_{\ell^\prime} (a_{m, \ell^\prime})^2 \leq 1$, and 
that $|S_{<j}| \leq N(j-1) \leq Nk$, by applying the Cauchy-Schwarz inequality we
get $\twonorm{v^<_m} \leq \epsilon_{j-1} \sqrt{Nk}$.

For $v^>_m$, we note that $\twonorm{ P v^>_m} \leq \lambda_{i_j +1}
\twonorm{v^>_m}$, since it only contains components corresponding to
eigenvectors with eigenvalues at most $\lambda_{i_j + 1}$. Also, note that
$\twonorm{v^>_m}^2 \leq \sum_{\ell^\prime \in S_{>j}} (a_{m, \ell^\prime})^2
\leq 1$.

We now complete the proof by induction. For, $j^\prime = 1, \ldots, j-1$,
suppose that all the eigenvectors corresponding to indices in $S_{j^\prime}$
have representations of the form in Equation~(\ref{eqn:eigvec-rep}) with
parameters $B_{j^\prime}$ and $\epsilon_{j^\prime}$ respectively.  Recall that
$\bar{a}_{\ell, m}$ is the element in position $(\ell, m)$ of $A^{-1}$, where
$A$ is the matrix defined as $a_{m, \ell} = \ip{g_m}{\nu_{\ell}}$ for $g_m \in
\G_j$ and $\ell \in S_j$. Now for any $\ell \in S_j$, we can define
$\bar{\nu}_{\ell}$ as follows (for the value $\tau_j$ to be specified later):
\begin{align}
\bar{\nu}_{\ell} &= \lambda_{\ell}^{-\tau_j} \sum_{m = 1}^{N_j} \bar{a}_{\ell, m}
	P^{\tau_j} \tilde{g}_m \label{eqn:nubar-define}
\intertext{Using Equation~(\ref{eqn:rep}) in the above equation, we get}
\bar{\nu}_\ell &= \lambda_{\ell}^{-\tau_j} \sum_{m = 1}^{N_j} \bar{a}_{\ell, m}
\sum_{\ell^\prime \in S_j} a_{m, \ell^\prime} P^{\tau_j} \nu_{\ell^\prime} +
\lambda_{\ell}^{-\tau_j} \sum_{m=1}^{N_j} P^{\tau_j} v^<_m  + \lambda_{\ell}^{-\tau}
\sum_{m=1}^{N_j} \bar{a}_{\ell, m} P^{\tau_j} v^>_m \nonumber \\
&= \lambda_{\ell}^{-\tau_j}  \sum_{\ell^\prime \in S_j}
\lambda_{\ell^\prime}^{\tau_j} \nu_{\ell^\prime} \sum_{m=1}^{N_j} \bar{a}_{\ell, m}
a_{m, \ell^\prime} + \lambda_{\ell}^{-\tau_j} \sum_{m=1}^{N_j} \bar{a}_{\ell, m}
P^{\tau_j} v^<_m + \lambda_{\ell}^{-\tau_j} \sum_{m=1}^{N_j} \bar{a}_{\ell, m}
P^{\tau_j} v^>_m \nonumber
\intertext{In the first term, we use the fact that $\sum_{m} \bar{a}_{\ell, m}
a_{m, \ell^\prime} = \delta_{\ell, \ell^\prime}$ by definition. Thus, the
first term reduces to $\nu_{\ell}$. We apply the triangle inequality to
get} 
\twonorm{\eta_{\ell}} = \twonorm{\nu_{\ell} - \bar{\nu}_{\ell}} &\leq
\lambda_{\ell}^{-\tau_j} \left\Vert \sum_{m = 1}^{N_j} \bar{a}_{\ell, m}
P^{\tau_j} v^<_m \right\Vert_2 + \lambda_{\ell}^{-\tau_j} \left\Vert
\sum_{m = 1}^{N_j} \bar{a}_{\ell, m} P^{\tau_j} v^>_m \right\Vert_2 \nonumber \\
&\leq \lambda_{\ell}^{-\tau_j} \sqrt{\sum_{m=1}^{N_j} (\bar{a}_{\ell, m})^2}
\cdot \sqrt{\sum_{m=1}^{N_j} \twonorm{P^{\tau_j} v^<_m}^2} + \lambda^{-\tau_j}
\sqrt{\sum_{m=1}^{N_j} (\bar{a}_{\ell, m})^2}  \cdot \sqrt{\sum_{m=1}^{N_j}
\twonorm{P^{\tau_j} v^>_m}^2} \label{eqn:nurep1}
\intertext{We use the fact that $\displaystyle \sqrt{\sum_{i=1}^m
(\bar{a}_{\ell, m})^2} \leq \frobenius{A^{-1}}$ and that $N_j \leq N$,
$\twonorm{v^<_m}^2 \leq Nk(\epsilon_{j-1})^2$ to simplify the above expression.
Furthermore, since $P$ has largest eigenvalue $1$, $\twonorm{P^{\tau_j}v} \leq
\twonorm{v}$ for any $v$. In the case of $v^>_m$, since the $\twonorm{v^>_m}
\leq 1$ and the largest eigenvalue in it is $\lambda_{i_j + 1}$,
$\twonorm{P^{\tau_j} v^<_m} \leq \lambda_{i_j + 1}^{\tau_j}$. Putting all these
together and simplifying the above expression we get}
\twonorm{\eta_{\ell}} &\leq \frobenius{A^{-1}} \sqrt{N}
\left(\lambda_{\ell}^{-\tau_j} \epsilon_{j-1} \sqrt{N k} +
\left(\frac{\lambda_{i_j + 1}}{\lambda_{\ell}}\right)^{\tau_j} \right)
\nonumber
\intertext{Finally, using the fact that $\lambda_{\ell} \geq \lambda_{i_{j}}$
(since $\ell \in S_j$), we have that $\lambda_{i_j  + 1}/\lambda_{\ell} \leq
\gamma$ and that $1/\lambda_{\ell} \leq \gamma^{-c}$. We also use the fact that
$\frobenius{A^{-1}} \leq \sqrt{N} \operatornorm{A^{-1}} \leq \sqrt{N}\alpha$.
Thus, we get}
\twonorm{\eta_{\ell}} &\leq \alpha N \left( \gamma^{-c \tau_j} \epsilon_{j-1}
\sqrt{Nk} + \gamma^{\tau_j} \right) \label{eqn:main-epsilon-step}
\end{align}

At this point we will deal with the base case $j = 1$ separately. In
Equation~(\ref{eqn:rep}) when $g_m \in \G_1$, $v^<_m = 0$, since the set $S_{<1}$
is empty. Thus, in Equation~(\ref{eqn:nurep1}), the first term is absent if we
are dealing with the case when $\ell \in S_1$, since all the $v^<_m$ in this
case are $0$. Thus, for $\ell \in S_1$, Equation~(\ref{eqn:main-epsilon-step})
reduces to:
\begin{align}
\twonorm{\eta_{\ell}} &\leq \alpha N \gamma^{\tau_1}
\end{align}
Thus, by choosing $\tau_1 = -\frac{\ln(N\alpha/\epsilon_1)}{\ln(\gamma)}$, we get
that for all $\ell \in S_1$, $\twonorm{\eta_{\ell}} \leq \epsilon_1$. Now, for
$j > 1$, we can find $\tau_j$ that minimizes the RHS of
Equation~(\ref{eqn:main-epsilon-step}) and this is given by $\tau_j = \frac{1}{1
+ c} \frac{\ln(\epsilon_{j-1} \sqrt{Nk})}{\ln(\gamma)}$. It is not hard to
calculate that in this case the RHS of Equation~\ref{eqn:main-epsilon-step}
exactly evaluates to $\epsilon_j$.

We now prove a bound on $B_j$. Again, we look at the base case separately, when
$j = 1$, $S_{<j} = \emptyset$ and so for the functions $g_m \in \G_1$ as in
Equation~(\ref{eqn:g-modify}), $\tilde{g}_m = g_m$. Thus, for $\ell \in S_1$, by looking at Equation~(\ref{eqn:nubar-define}), we can define: $\beta^{\ell}_{\tau_1, m} = \lambda_{\ell}^{-\tau_1} \bar{a}_{\ell, m}$ for $m \in \G_1$ and the remaining $\beta^{\ell}_{t, m}$ values are set to $0$. Thus,
\begin{align}
\sum_{t, m} |\beta^{\ell}_{t, m}| &\leq \lambda_{\ell}^{-\tau_1} \sum_{m =
1}^{N_1} |\bar{a}_{\ell, m}| \leq \gamma^{-c \tau_1} N \alpha
\end{align}
Above we used the fact that $\lambda_{\ell} \geq \lambda_{i_k} \geq \gamma^{c}$ and that $|\bar{a}_{\ell, m}| \leq \operatornorm{A^{-1}}$. But, the RHS above is exactly the quantity $B_1$ we defined earlier. 

Next, we consider the case of $j > 1$ and we start from Equation~(\ref{eqn:nubar-define}).
\begin{align}
	\bar{\nu}_{\ell} &= \lambda_{\ell}^{-\tau_j} \sum_{m = 1}^{N_j} \bar{a}_{\ell, m}
	P^{\tau_j} \tilde{g}_m \nonumber \\
	&= \lambda_{\ell}^{-\tau_j} \sum_{m = 1}^{N_j} \bar{a}_{\ell, m} P^{\tau_j} \left( 
	g_m - \sum_{\ell^\prime \in S_{<j}} a_{m, \ell^\prime}
	\bar{\nu}_{\ell^\prime} \right) \nonumber \\
	&= \lambda_{\ell}^{-\tau_j} \sum_{m = 1}^{N_j} \bar{a}_{\ell, m} P^{\tau_j} \left( g_m - \sum_{\ell^\prime \in S_{< j}} a_{m, \ell^\prime} \sum_{t, m^\prime} \beta^{\ell^\prime}_{t, m^\prime} P^t g_{m^\prime} \right) \nonumber \\
	&= \lambda_{\ell}^{-\tau_j} \sum_{m = 1}^{N_j} \bar{a}_{\ell, m} \left( P^{\tau_j} g_m - \sum_{\ell^\prime \in S_{< j}} a_{m, \ell^\prime} \sum_{t, m^\prime} \beta^{\ell^\prime}_{t, m^\prime} P^{t + \tau_j} g_{m^\prime} \right) \label{eqn:expanded-expression}
	\intertext{If the above, expression is re-written to be of the form,}
	\bar{\nu}_{\ell} = \sum_{t, m} \beta^{\ell}_{t, m} P^t g_m, \nonumber
	\intertext{we can get a bound on $\sum_{t, m} |\beta^{\ell}_{t, m}|$ as follows:}
	\sum_{t, m} |\beta^{\ell}_{t, m}| &\leq \gamma^{-c \tau_j} \left(\sum_{m=1}^{N_j} |\bar{a}_{\ell, m}| \right) \cdot \left(1 + B_{j-1} \sum_{\ell^\prime \in S_{<j}} |a_{m, \ell^\prime}|\right) \nonumber 
	\intertext{Above, we use the fact that for $\ell^\prime \in S_{<j}$, $\sum_{t, m} |\beta^{\ell^\prime}_{t, m}| \leq B_{j-1}$. Also, note that $\sum_{m=1}^{N_j} |\bar{a}_{\ell, m}| \leq N \alpha$ and $\sum_{\ell^\prime \in S_{<j}} |a_{m, \ell^\prime}| \leq \sqrt{Nk}$ (since $\sum_{\ell^\prime} (a_{m, \ell^\prime})^2 \leq 1$ for all $m$), so we have}
	\sum_{t, m} |\beta^{\ell}_{t, m}| &\leq (\epsilon_{j-1} \sqrt{Nk})^{-\frac{c}{1 + c}} N \alpha (1 + \sqrt{Nk} B_{j-1}) \nonumber  \\
	&\leq 2 \sqrt{Nk} N\alpha B_{j-1} (\epsilon_{j-1} \sqrt{Nk})^{-\frac{c}{1 +
	c}}  \nonumber
\end{align}

We observe that the expression on the RHS above is exactly the value $B_j$ given by the recurrence relation in Equation~(\ref{eqn:B-recurrence}). 

Finally, by observing the RHS of Equation~(\ref{eqn:nubar-define}) we notice
that the maximum power $t$, for which $\beta^{\ell}_{t, m}$ is non-zero for any
$\ell, m$ is $\sum_{i=1}^k \tau_i$. Thus, the proof is complete by setting
$\tau_{\max} = \sum_{j = 1}^k \tau_j$.

\end{proof}

\subsection{Proof of Theorem~\ref{thm:agnostic-learning}}
\label{app:eigen2}

\begin{proof}
Let $f \in F$ be the target function and for any $\ell$, let $\hat{f}_{\ell} =
\ip{f}{\nu_{\ell}}$ denote the \emph{Fourier} coefficients of $f$. Then the
condition in Theorem~\ref{thm:agnostic-learning} states that $\sum_{\ell >
\ell^*(\epsilon)} \hat{f}_{\ell}^2 \leq \epsilon^2/4$. 

First, we appeal to Theorem~\ref{thm:spectrum}. In the rest of this proof, we
assume that for all $\ell \leq \ell^*$, there exist $\beta^{\ell}_{t, m}$ such that
\begin{align}
	\nu_{\ell} = \sum_{t, m} \beta^{\ell}_{t, m} P^t g_m + \eta_{\ell}, \nonumber
\end{align}
where $g_m \in \G$, $\twonorm{\eta_{\ell}} \leq \epsilon_1$. Furthermore, let
$B$ and $\taumax$ be as given by the statement of the theorem.

We first look closely at $P^t g_m$, since $P$ is an $|X| \times |X|$ matrix and
$g_m : X \rightarrow \reals$ a function, $P^t g_m$ is also a function from $X
\rightarrow \reals$. For $x \in X$, let $\one_x$ denote the indicator function
of the point $x$ (it may be viewed as a vector that is $0$ everywhere, except
in position $x$ where it has value $1$). Then, we have 
\begin{align}
	(P^t g_m)(x) &= \one_x^{\transpose} P^t g_m = \E_{y \sim P^t(x, \cdot)} [g_m(y)] \label{eqn:sampling}
\end{align}
Notice that the quantity on the RHS above can be estimated by sampling. Thus,
with black-box access to the oracle $\OS(\cdot)$ and $g_m$, we can estimate
$(P^t g_m)(x)$. This is exactly what is done in~(\ref{eqn:phi-define}) in the
algorithm in Figure~\ref{fig:learning-algorithm}. Also, since $\infinitynorm{g}
\leq 1$, it is also the case that $\infinitynorm{P^t g_m} \leq 1$. Thus, by a
standard Chernoff-Hoeffding bound, if we set the input parameter $T = \log(
\taumax \cdot |X| \cdot |\G|/ \delta)/\epsilon_2^2$, with probability at least $1
- \delta$, it holds for every $x \in X$, for every $t < \taumax$ and every $g_m
\in \G$, that $|\phi_{t, m}(x) -  (P^t g_m)(x)| \leq \epsilon_2$. For the rest
of this proof, we will treat the functions $\phi_{t,m}(x)$ as deterministic
(rather than randomized) for simplicity. (This can be easily arranged by taking
a sufficiently long random string used to simulate the Markov chain and treating
it as advice.)

Now, consider the following:
\begin{align}
&\E\left[\left(f(x) - \sum_{\ell \leq \ell^*} \hat{f}_{\ell} \sum_{t, m}
\beta^{\ell}_{t, m} \phi_{t, m}(x)\right)^2\right]  \nonumber \\
&~~~~~~\leq 2 \E\left[\left(f(x) - \sum_{\ell \leq \ell^*} \hat{f}_{\ell}
\nu_{\ell}(x)\right)^2\right] + 2 \E\left[\left(\sum_{\ell} \hat{f}_{\ell}
\left(\nu_{\ell}(x) - \sum_{t, m} \beta^{\ell}_{t, m} \phi_{t, m}(x)\right)
\right)^2\right] \label{eqn:learnproof1}
\intertext{Note that the first term above is at most $\epsilon$. We will now
bound the second term. (Below $\bar{\nu}_{\ell}$ is as defined in
Equation~(\ref{eqn:nubar-define}).)}
&\E\left[\left(\sum_{\ell} \hat{f}_{\ell} \left(\nu_{\ell}(x) - \sum_{t, m}
\beta^{\ell}_{t, m} \phi_{t, m}(x)\right)\right)^2\right] \nonumber \\ 
&~~~~~~\leq 2 \E\left[\left(\sum_{\ell \leq \ell^*} \hat{f}_{\ell}( \nu(x) -
\bar{\nu}(x))\right)^2\right] + 2\E\left[\left(\sum_{\ell} \hat{f}_{\ell}
\left(\sum_{t, m} \beta^{\ell}_{t, m} ((P^t g_m)(x) - \phi_{t,m}(x))\right)
\right)^2 \right] \nonumber  \\
&~~~~~~\leq 2 \sqrt{\sum_{\ell \leq \ell^*} (\hat{f}_{\ell})^2} \cdot \sqrt{\sum_{\ell
\leq \ell^*} \twonorm{\eta_{\ell}}^2} + 2 \sqrt{\sum_{\ell \leq \ell^*}
(\hat{f}_{\ell})^2} \cdot \sqrt{\sum_{\ell < \ell^*} \left\Vert \sum_{t, m}
\beta^{\ell}_{t, m} (P^tg_m - \phi_{t, m}) \right\Vert^2_2 } \nonumber
\intertext{Next we use the following facts, $\sum_{\ell \leq \ell^*}
(\hat{f}_\ell)^2 \leq 1$, $\ell^* \leq Nk$ and $\twonorm{\eta_{\ell}} \leq
\epsilon_1$. Also for the very last term, the fact that $\sum_{t, m}
|\beta^{\ell}_{t, m}| \leq B$ and $\forall x, |(P^tg_m)(x) - \phi_{t, m}(x)| \leq
\epsilon_2$, imply that $\Vert \sum_{t, m} \beta^{\ell}_{t, m} (P^tg_m -
\phi_{t, m}) \Vert_2 \leq B \epsilon_2$. Putting everything together we get}
&\E\left[\left(\sum_{\ell} \hat{f}_{\ell} \left(\nu_{\ell}(x) - \sum_{t, m}
\beta^{\ell}_{t, m} \phi_{t, m}(x)\right)\right)^2\right] \leq 2 (\sqrt{Nk
\epsilon_1} + \sqrt{BNk \epsilon_2}) \label{eqn:learnproof2}
\intertext{Finally, substituting~(\ref{eqn:learnproof2}) back into~(\ref{eqn:learnproof1}), we get}
&\E\left[\left(f(x) - \sum_{\ell \leq \ell^*} \hat{f}_{\ell} \sum_{t, m}
\beta^{\ell}_{t, m} \phi_{t, m}(x)\right)^2\right] \leq 2(\frac{\epsilon^2}{4} + \sqrt{Nk \epsilon_1} + \sqrt{BNk \epsilon_2}) \\
\intertext{By choosing $\epsilon_1 = \epsilon^2/(64 Nk)$ and $\epsilon_2 =
\epsilon^2/(64 BNk)$ we get that the quantity is in fact at most
$\epsilon^2$. Thus, we get that}
&\E\left[\left|f(x) - \sum_{\ell \leq \ell^*} \hat{f}_{\ell} \sum_{t, m}
\beta^{\ell}_{t, m} \phi_{t, m}(x)\right|\right] = \epsilon \label{eqn:learnproof3}
\end{align}

Thus, we have essentially shown that $\{ \phi_{t, m} \}$ can be used as a
suitable feature space and there is a linear form in this feature space that is
a good $L_1$ approximation to $f$. This is sufficient for agnostic learning as
was shown by~\cite{KKMS:2005}. Note that the sum of coefficients on the features
is bounded by $B \sqrt{Nk}$ (since $B$ is a bound on $\sum_{t, m}
|\beta^{\ell}_{t,m }|$ and $\sum_{\ell < \ell^*} |\hat{f}_{\ell}| \leq \sqrt{Nk}$).
Thus, in the algorithm in Figure~\ref{fig:learning-algorithm}, we may set the
parameters $\taumax$ (as given by Theorem~\ref{thm:spectrum}), $W = B \sqrt{Nk}$
and $T = \log(\taumax \cdot |X| \cdot |\G|/\delta)/\epsilon_2^2$. The sample complexity is
polynomial in $W$, $1/\epsilon$ as follows from standard generalization bounds
(see for example~\cite{KST:2008}) and the running time of the algorithm is
polynomial in $|\G|, T, \tau_{\max}, W, \frac{1}{\epsilon}$. The bounds given in
the statement of the theorem follow from observing the values of the above
quantity in the statement of Theorem~\ref{thm:spectrum}.
\end{proof}

\subsection{Proof of Theorem~\ref{thm:noise-sensitivity}}

\label{app:eigen3}

\begin{proof} The proof  follows the standard proofs of these types of results.
Let $t = -\frac{1}{\ln(\rho)}$
\begin{align}
\NS_t(f) &= \frac{1}{2} - \frac{1}{2} \ip{f}{P^t f} \nonumber  \\
&= \frac{1}{2} - \frac{1}{2} \left(\sum_{\ell} \lambda_{\ell}^t
\hat{f}^2_{\ell}\right) \nonumber \\
&\geq \frac{1}{2} - \frac{1}{2} \left(\sum_{\ell \leq \ell^*} \hat{f}^2_{\ell} +
\sum_{\ell > \ell^*} \rho^t \hat{f}^2_{\ell}\right) \nonumber
\intertext{Using the fact that $\sum_{\ell \leq \ell^*} \hat{f}^2_\ell = 1 -
\sum_{\ell > \ell^*} \hat{f}^2_\ell$ (since $f$ is boolean) and rearranging terms,
we get}
\sum_{\ell > \ell^*} \hat{f}^2_{\ell} &\leq \frac{1}{1 - \rho^t}  \NS_t(f)
\end{align}
Then substituting the value for $t$ completes the proofs. 
\end{proof}

\subsection{Proof of Proposition~\ref{prop:one}}

\label{app:eigen4}

\begin{lemma} \label{lemma:new_one} For any positive integer $\Delta$, there exists
$\beta(\Delta)$, such that for all graphs $G$ of maximum degree bounded by
$\Delta$, and for all ferromagnetic Ising models with $\beta < \beta(\Delta)$,
the following holds. If $f = \sign(\sum_{i} w_i x_i)$, then for all $t \geq n$
it holds that, 
\[ 1 - 2 \NS_t(f) \geq \delta^{t/n} \]
for some fixed $\delta > 0$ depending only on $\Delta$ and $\beta(\Delta)$.
\end{lemma}

Note that the above lemma only proves that majorities are somewhat noise stable.
While one expects that if $t$ is a very small fraction on $n$, majorities are
very noise stable, our proof is not strong enough to prove that. 

For the proof we will need to use the following well known result which goes
back to \citet{DobrushinShlosman:85}. The proof also follows easily from the random cluster representation of the Ising model.

\begin{lemma} \label{lemma:DS85} For every $\Delta$ and $\eta > 0$, there
exists a $\beta(\Delta, \eta) > 0$ such that for all graphs $G$ of maximum degree
bounded by $\Delta$ and for all Ising models where $\beta \leq \beta(\Delta,
\eta)$, it holds that under the stationary measure for any $i$ and any subset $S$ of nodes, 
\[
\E[x_i ~|~ x_S] \leq \eta^{d(i, S)} \]
In particular, for every $i$ and $j$,
	\[ \E[x_i x_j] \leq \eta^{d(i, j)}, \]
where $d(i, j)$ denotes the graph distance between $i$ and $j$.
\end{lemma}

We will need a few corollaries of the above lemma.

\begin{lemma} \label{lemma:new_two} If $\beta < \beta(\Delta, 1/(10 \Delta) )$, then
for every set $A$ and any weights $w_i$, it holds that if $f(x) = \sum_{i} w_i
x_i$, then:
	\begin{enumerate}
		\item $\frac{4}{5} \sum_{i} w_i^2 \leq \E[f(x)^2] \leq \frac{6}{5} \sum_{i} w_i^2$
		\item $\E[(f(x))^4] \leq 10 \left( \sum_{i} w_i^2 \right)^2$
	\end{enumerate}
\end{lemma}
\begin{proof} For the first claim note that
	\begin{align*} 
		\E[f(x)^2] &= \sum_{i} w_i^2 + \sum_{i \neq j} w_i w_j \E[x_i x_j] \\
		&= \sum_{i} w_i^2 + \sum_{d=1}^n \sum_{i, j : d(i, j) = d} w_i w_j \E[x_i x_j]
		\intertext{Choose $\eta = 1/(10 \Delta)$ in Lemma~\ref{lemma:DS85}, and suppose that $\beta < \beta(\Delta, 1/(10 \Delta))$, then we have that for all $i \neq j$,}
		\E[x_i x_j] &\leq (10 \Delta)^{-d(i, j)}. 
		\intertext{We may thus bound,}
	\left| \sum_{i, j : d(i, j) = d} w_i w_j \E[x_i x_j] \right| &\leq (10 \Delta)^{-d} \sum_{i, j: d(i, j) = d} |w_i w_j| 
\end{align*}
For each $i$, let $v^i_{1}, \ldots , v^i_{\Delta^d}$ be all the nodes that are at distance $d$ from $i$, where if the actual number of such nodes is less than $\Delta^d$, we set the remaining $v^i_j = i$. Then, by applying the Cauchy Schwarz inequality, we can write:
\begin{align*}
	\sum_{i, j: d(i, j) = d} |w_i w_j| &\leq \sum_{i} \sum_{j = 1}^{\Delta^d} |w_i w_{v^i_j}| \leq \Delta^d \sum_{i} w_i^2
	\intertext{So, adding up over all $d$, we obtain,}
	|\E[f(x)^2] - \sum_{i} w_i^2| &\leq \sum_{i} w_i^2 \sum_{d=1}^n 10^{-d} \leq \frac{1}{5} \sum_{i} w_i^2
\end{align*}

This completes the proof of the first part of the lemma.

The second part is proved analogously, however, the calculations are a bit more involved since it involves terms corresponding to four nodes at a time.



\end{proof}

We can now complete the proof of Lemma~\ref{lemma:new_one}.

\begin{proof}[Proof of Lemma~\ref{lemma:new_one}] From the Fourier expression
of noise-sensitivity (see Eq.~\ref{eqn:noise-sensitivity}) and Jensen's inequality, it is clear
that if $a > 1$, then 
\begin{align*}
	1 - 2 \NS_{at}(f) \geq (1 - 2 \NS_t(f))^a
\end{align*}
Therefore it suffices to prove the claim when $t = cn$ for some small constant $c$ (which may depend on $\Delta$). Our goal is therefore to show that: 
\begin{align*}
	1 - 2 \NS_{cn}(f) \geq \delta > 0
\end{align*}
where $\delta$ is a parameter that depends only on $\Delta$ (but not $n$). To prove this let $X_1, \ldots, X_n$ be the system at time $0$ and let $Y_1, \ldots, Y_n$ be the system at time $t = cn$. Let $A \subset [n]$ be the random subset of spins that have not been updated from time $0$ to time $t$. Then, the noise sensitivity is: 
\begin{align*} 
	\NS_{\delta}(f) &= \Pr\left[ \sign(\sum_{i \in A} w_i X_i + \sum_{i \not\in A} w_i X_i) \neq \sign(\sum_{i \in A} w_i X_i + \sum_{i \not\in A} w_i Y_i) \right] \\
	&\leq 2 \Pr\left[\sign(\sum_{i \in A} w_i X_i) \neq \sign(\sum_{i = 1}^n w_i X_i) \right],
\end{align*}
where the last inequality uses the fact that $X_i$, $i \not \in A$ and $Y_i$, $i
\not\in A$ are identically distributed given $A$ and $X_i$, $i \in A$ (the
distribution for both is just the conditional distribution given $x_i$ for $i
\in A$).
 
Let $W = \sum_{i} w_i^2$. By Markov's inequality, it follows that for $c$ chosen small enough with probability at least $9/10$ (over the random choice of $A$), we have:
\begin{align}
	\sum_{i \not\in A} w_i^2 &\leq 10^{-6} \cdot {W} \nonumber
	\intertext{From now on, we will condition on the event that $\sum_{i \in A} w_i^2 \geq (1 - 10^{-6})W$, which we denote by $\scE$. Under this conditioning, from Lemma~\ref{lemma:new_two}, it follows that}
	\E\left[\left(\sum_{i \in A} w_i X_i \right)^2 \right] &\geq \frac{3}{5} W \label{eqn:avocado}
\end{align}
 
Moreover, we claim that with probability at least $1/40$ (conditioned on the event above), it holds that:
\begin{align*}
	\left( \sum_{i \in A} w_i X_i \right)^2 \geq \frac{W}{10}
\end{align*}
Let $\rho$ be the (conditioned on $\scE$) probability of the above event, which we denote by $\scE^\prime$. Note that~(\ref{eqn:avocado}) implies that:
\begin{align*}
  	\E\left[\left(\sum_{i \in A} w_i X_i \right)^2 ~\left|~  \scE^\prime \right. \right] \geq \frac{W}{2\rho}
\end{align*}
But, then we use part two of Lemma~\ref{lemma:new_two} to conclude that $\rho \geq 1/40$; if not, we can derive a contradiction as follows.
\begin{align*}
 	\E\left[\left(\sum_{i  \in A} w_i X_i \right)^4\right] &\geq
 	\E\left[\left(\sum_{i \in A} w_i X_i \right)^4 ~\left|~ \scE^\prime \right.
 	\right] \cdot \rho \\
 	&\geq \E\left[\left(\sum_{i \in A} w_i X_i \right)^2 ~\left|~ \scE^\prime
 	\right. \right]^2 \cdot \rho > 10 W^2
\end{align*}

Also, conditioned on the event $\scE$, by Markov's Inequality, we have:
\begin{align*}
	\Pr\left[\left(\sum_{i \not\in A} w_i X_i \right)^2 \geq \tfrac{W}{100} \right] &\leq 10^{-4} \\
	\Pr\left[\left(\sum_{i \not\in A} w_i Y_i \right)^2 \geq \tfrac{W}{100} \right] &\leq 10^{-4}
\end{align*}
 
Thus, conditioned on $\scE$, by a union bound, we have that with probability at least $3/4$:
\begin{align*}
\sign\left(\sum_{i = 1}^n w_i X_i \right) = \sign\left(\sum_{i=1}^n w_i Y_i \right) = \sign\left(\sum_{i \in A} w_i X_i \right)
\end{align*}
 
To conclude the proof, we show that when $\scE$ does not hold, the probability that
\begin{align*}
	\sign\left(\sum_{i=1}^n w_i X_i \right) = \sign\left(\sum_{i=1}^n w_i Y_i \right)
\end{align*}
is at least $1/2$. In fact, we show this conditioned on any $A$ and any values of the random variables $X_i, i \in A$. Note that conditioned on $A$ and $X_i \in A$, the random variables $X_i$ and $Y_i$ for $i \not\in A$ are positively correlated. (Also, $(X_i)_{i \not \in A}$ and $(Y_i)_{i \not \in A}$ are identically distributed.) Thus, if we denote by 
\begin{align*}
 		p_A = \Pr \left[ \sign\left(\sum_{i=1}^n w_i X_i\right) \neq \sign\left(\sum_{i \in A} w_i X_i\right)\right] = \Pr \left[ \sign\left(\sum_{i=1}^n w_i Y_i\right) \neq \sign\left(\sum_{i \in A} w_i X_i\right)\right]
\end{align*}
Then, using the FKG inequality, we see that conditioned on the event $\scE$ not occurring, 
 	\begin{align*}
 		\Pr\left[\sign\left(\sum_{i=1}^n w_i X_i\right) \neq \sign\left(\sum_{i=1}^n w_i Y_i \right) \right] \leq 2 p_A \cdot (1 - p_A) \leq \frac{1}{2}
 	\end{align*}
 
This concludes the proof.
 
\end{proof}



\end{document}